\documentclass{article}

    \usepackage[final]{neurips_2021}

\usepackage[utf8]{inputenc} 
\usepackage[T1]{fontenc}    
\usepackage[colorlinks]{hyperref}       
\usepackage{url}            
\usepackage{booktabs}       
\usepackage{amsfonts}       
\usepackage{nicefrac}       
\usepackage{microtype}      
\usepackage{xcolor}         
\usepackage{multibib}

\hypersetup{
     citecolor    = red,
     linkcolor = blue,
     urlcolor = blue,
     linkbordercolor=blue
}

\title{Conditional Generation Using Polynomial Expansions}

\usepackage{times}
\usepackage{epsfig}
\usepackage{graphicx}
\usepackage{amsmath}
\usepackage{amssymb}
\usepackage{xcolor}
\usepackage{dsfont}

\usepackage[utf8]{inputenc} 
\usepackage[T1]{fontenc}    
\usepackage{url}            
\usepackage{booktabs}       
\usepackage{amsfonts}       
\usepackage{nicefrac}       
\usepackage{microtype}      
\usepackage{graphicx}
\usepackage{mathabx}
\usepackage{longtable}
\usepackage{amsthm}   
\usepackage{stackrel}  

\usepackage{bm} 
\usepackage{subfig}

\usepackage{float}  
\usepackage{multirow}
\usepackage{xcolor,colortbl} 
\usepackage{color}  
\usepackage[skip=2pt]{caption}  
\usepackage{wrapfig} 
\usepackage{pifont}  
\usepackage{amsbsy}  
\usepackage{selectp}

\providecommand{\realnum}{\mathbb{R}}
\providecommand{\naturalnum}{\mathbb{N}}
\providecommand{\bmcal}[1]{\bm{\mathcal{#1}}}

 \providecommand{\matnot}[1]{_{[{#1}]}}  
 \providecommand{\invar}{z}  
\providecommand{\binvar}{\bm{\invar}}  
\providecommand{\outvar}{x}  
\providecommand{\boutvar}{\bm{\outvar}}  
\providecommand{\knum}{\mathbb{K}}
\providecommand{\sbinvar}[1]{\binvar_{\scalebox{.6}{#1}}} 
 
\providecommand{\citep}{\cite} 
\providecommand{\citet}{\cite}

\newtheorem{lemma}{Lemma}
\newtheorem{claim}{Claim}
\providecommand{\supplementary}{supplementary}

\newcommand{\modelnamePI}{$\Pi$-Net}
\newcommand{\modelnameMVP}{CoPE}
\newcommand{\modelnameMVPsngan}{SNGAN-\modelnameMVP}

\newcommand{\modelnamespade}{SPADE}

\newcommand{\modeltwo}{Nested-\modelnameMVP} 

\newcommand{\xmark}{\textcolor{red}{\ding{55}}}   
\definecolor{mygreen}{RGB}{30, 180, 50}
\newcommand{\colorcheck}{\textcolor{mygreen}{\pmb{\checkmark}}}

\author{
Grigorios G Chrysos\\
EPFL, Switzerland\\
\texttt{grigorios.chrysos@epfl.ch}\\
\And
Markos Georgopoulos\\
Imperial College London, UK\\
\texttt{m.georgopoulos@imperial.ac.uk}\\
\And
Yannis Panagakis\\
University of Athens, GR\\
\texttt{yannisp@di.uoa.gr}\\
}

\begin{document}
\maketitle

\begin{abstract}

Generative modeling has evolved to a notable field of machine learning. Deep polynomial neural networks (PNNs) have demonstrated impressive results in unsupervised image generation, where the task is to map an input vector (i.e., noise) to a synthesized image. However, the success of PNNs has not been replicated in conditional generation tasks, such as super-resolution. Existing PNNs focus on single-variable polynomial expansions which do not fare well to two-variable inputs, i.e., the noise variable and the conditional variable. In this work, we introduce a general framework, called \modelnameMVP{}, that enables a polynomial expansion of two input variables and captures their auto- and cross-correlations. We exhibit how \modelnameMVP{} can be trivially augmented to accept an arbitrary number of input variables. \modelnameMVP{} is evaluated in five tasks (class-conditional generation, inverse problems, edges-to-image translation, image-to-image translation, attribute-guided generation) involving eight datasets. The thorough evaluation suggests that \modelnameMVP{} can be useful for tackling diverse conditional generation tasks. The source code of \modelnameMVP{} is available at \url{https://github.com/grigorisg9gr/polynomial_nets_for_conditional_generation}. 

\end{abstract}

\section{Introduction}
\label{sec:mvp_intro}

Modelling high-dimensional distributions and generating samples from complex distributions are fundamental tasks in machine learning. Among prominent generative models, StyleGAN~\citep{karras2018style} has demonstrated unparalleled performance in unsupervised image generation. Its success can be attributed to the higher-order correlations of the input vector $\binvar$ captured by the generator. As \citet{chrysos2019polygan} argue, StyleGAN\footnote{\label{foot:mvp_generator}This work focuses on the generator network; any reference to StyleGAN refers to its generator.} is best explained as a deep polynomial neural network (PNN). PNNs have demonstrated impressive generation results in faces, animals, cars~\citep{karras2020analyzing}, paintings, medical images~\citep{karras2020training}. 
Nevertheless, PNNs have yet to demonstrate similar performance in conditional generation tasks, such as super-resolution or image-to-image translation.

In contrast to unsupervised generators that require a single-variable input $\binvar$, in conditional generation (at least) two inputs are required: i) one (or more) conditional variables $\bm{c}$, e.g., a low-resolution image, and ii) a noise sample $\binvar$. A trivial extension of PNNs for conditional generation would be to concatenate all the input variables into a fused variable. The fused variable is then the input to the single-variable polynomial expansion of PNNs. However, the concatenation reduces the flexibility of the model significantly.
For instance, concatenating a noise vector and a vectorized low-resolution image results in sub-optimal super-resolution, since the spatial correlations of the input image are lost in the vectorization. Additionally, the concatenation of the vectorized conditional variable $\bm{c}$ and $\binvar$ leads to a huge number of parameters when we use a fully-connected layer as typically done in the input of StyleGAN, especially when $\bm{c}$ depicts an image.

In this work, we introduce a framework, called \emph{\modelnameMVP{}}, for conditional data generation. \modelnameMVP{} resorts to multivariate polynomials that capture the higher-order auto- and cross-correlations between the two input variables. 
By imposing a tailored structure in the higher-order correlations, we obtain an intuitive, recursive formulation for \modelnameMVP. 
The formulation enables different constraints to be applied to each variable and its associated parameters. 
In \modelnameMVP, different architectures can be defined simply by changing the recursive formulation. 
Our contributions can be summarized as follows:
\vspace{-1mm}
\begin{itemize}
    \setlength\itemsep{0.1em}
    \item We introduce a framework, called \modelnameMVP{}, that expresses a high-order, multivariate polynomial for conditional data generation. We exhibit how \modelnameMVP{} can be applied on diverse conditional generation tasks. 
    \item We derive two extensions to the core two-variable model: a) we augment the formulation to enable an arbitrary number of conditional input variables, b) we design different architectures that arise by changing the recursive formulation. 
    \item \modelnameMVP{} is evaluated on  \emph{five different tasks} (class-conditional generation, inverse problems, edges-to-image translation, image-to-image translation, attribute-guided generation); overall \emph{eight datasets} are used for the thorough evaluation.
\end{itemize}{}
\vspace{-1mm}
The diverse experiments suggest that \modelnameMVP{} can be useful for a variety of conditional generation tasks, e.g., by defining task-specific recursive formulations. To facilitate the reproducibility, the source code is available at \url{https://github.com/grigorisg9gr/polynomial_nets_for_conditional_generation}.

 \section{Related work}
\label{sec:mvp_related}
Below, we review representative works in conditional generation and then we summarize the recent progress in multiplicative interactions (as low-order polynomial expansion).

\subsection{Conditional generative models}
\label{ssec:mvp_related_conditional}
The literature on conditional generation is vast. The majority of the references below focus on Generative Adversarial Networks (GANs)~\citep{goodfellow2014generative} since GANs have demonstrated the most impressive results to date, however similar methods can be developed for other generative models, such as Variational Auto-encoders (VAEs)~\citep{kingma2013auto}. 
Four groups of conditional generative models are identified below based on the type of conditional information. 

The first group is the \emph{class-conditional generation}~\citep{miyato2018spectral, brock2019large, kaneko2019label}, where the data are divided into discrete categories, e.g., a cat or a ship. During training a class label is provided and the generator should synthesize a sample from that category. One popular method of including class-conditional information is through  conditional normalization techniques~\citep{dumoulin2016learned, de2017modulating}. An alternative way is to directly concatenate the class labels with the input noise; however, as ~\citet{odena2017conditional} observe, this model does not scale well to a hundred or a thousand classes. 

The second group is \emph{inverse problems} that have immense interest for both academic and commercial reasons~\citep{sood2018application, you2019ct}. The idea is to reconstruct a latent signal, when corrupted measurements are provided~\citep{ongie2020deep}. Well-known inverse problems in imaging include super-resolution, deblurring, inpainting. Before the resurgence of deep neural networks, the problems were tackled with optimization-based techniques~\citep{chan2006optimization, levin2009understanding}. The recent progress in conditional generation has fostered the interest in inverse problems~\citep{ledig2016photo, pathak2016context, huang2017beyond}.  
A significant challenge that is often overlooked in the literature~\citep{ledig2016photo, yu2018generative} is that there are many possible latent signals for each measurement. For instance, given a low-resolution image, we can think of several high-resolution images that when down-sampled provide the same low-resolution image.

Another group is \emph{image-to-image translation}. The idea is to map an image from one domain to an image to another domain, e.g., a day-time image to a night-time image. The influential work of \citet{isola2016image} has become the reference point for image-to-image translation. 
Applications in conditional pose generation~\citep{ma2017pose, siarohin2018deformable}, conditional video generation~\citep{wang2018video} or generation from semantic labels~\citep{wang2018high} have appeared. Despite the success, converting the mapping from one-to-one to one-to-many, i.e., having multiple plausible outputs for a single input image, has required some work~\citep{zhu2017toward, huang2018multimodal}. A more dedicated discussion on diverse generation is deferred to sec.~\ref{sec:mvp_difference_from_other_diverse_generation_techniques_suppl}.  

The fourth group uses multiple conditional variables for generation, e.g., attribute-guided generation~\citep{choi2018stargan}. These methods include significant engineering (e.g., multiple discriminators~\citep{xu2017learning}, auxiliary losses). 
The influential InfoGAN~\citep{chen2016infogan} explicitly mentions that without additional losses the generator is `free to ignore' the additional variables.

Each technique above is typically applied to a single group of conditional generation tasks, while our goal is to demonstrate that \modelnameMVP{} can be applied to different tasks from these groups.

\subsection{Multiplicative interactions}
\label{ssec:mvp_related_multiplicative}

Multiplicative connections have long been adopted in machine learning~\citep{shin1991pi, hochreiter1997long, bahdanau2014neural, rendle2010factorization}. The idea is to combine the inputs through elementwise products or other diagonal forms. 
\citet{jayakumar2020Multiplicative} prove that second order multiplicative operators can represent a greater class of functions than classic feed-forward networks. Even though we capitalize on the theoretical argument, our framework can express any higher-order correlations while the framework of \citet{jayakumar2020Multiplicative} is limited to second order interactions.

Higher-order correlations have been studied in the tensor-related literature~\citep{kolda2009tensor, debals2017concept}. However, their adaptation in modern deep architectures has been slower. \modelnamePI~\citep{chrysos2020poly} resorts to a high-order polynomial expansion for mapping the input $\binvar$ to the output $\boutvar = G(\binvar)$. \modelnamePI{} focuses on a single-variable polynomial expansion; an in-depth difference of our work with \modelnamePI{} can be found in sec.~\ref{sec:mvp_difference_from_pinet_suppl} in the \supplementary. 
Two efforts on conditional generation which can be cast as polynomial expansions are \modelnamespade~\citep{park2019semantic} and sBN~\citep{chen2018self}. 
\modelnamespade{} can be interpreted as a single-variable polynomial expansion with respect to the conditional variable $\bm{c}$. \modelnamespade{} does not capture the higher-order cross-correlations between the input variables. 
Similarly, sBN can be interpreted as a polynomial expansion of the two variables for class-conditional generation. However, \modelnamespade{} and sBN do not use the product of polynomial formulation, which enables high-order expansions without increasing the number of layers~\citep{chrysos2020poly}. 
Importantly, \modelnamespade{} and sBN are constructed for specific applications (i.e., semantic image generation and unsupervised/class-conditional generation respectively) and it remains unclear whether a PNN can effectively tackle a general-purpose conditional generation task.

\begin{table*}[tb]
\caption{Comparison of polynomial neural networks (PNNs). Even though the architectures\textsuperscript{\ref{foot:mvp_generator}} of \citet{karras2018style, chen2018self, park2019semantic} were not posed as polynomial expansions, we believe that their success can be (partly) attributed to the polynomial expansion (please check sec.~\ref{sec:mvp_difference_from_other_polynomial_networks_suppl} for further information). \modelnamePI{} and StyleGAN are not designed for conditional data generation. In practice, learning complex distributions requires high-order polynomial expansions; this can be effectively achieved with products of polynomials as detailed in sec.~\ref{ssec:mvp_method_extensions}. Only \modelnamePI{} and \modelnameMVP{} include such a formulation. The columns on discrete and continuous variable refer to the type of conditional variable the method was originally proposed on, e.g., sBN was only tried on class-conditional generation. Additionally, the only work that enables multiple conditional variables (and includes related experiments) is the proposed \modelnameMVP. }
\scalebox{0.97}{
\centering
         \begin{tabular}{|c | c | c| c| c|} 
             \hline
             \multicolumn{5}{|c|}{Attributes of polynomial-like networks.}\\ 
             \hline
             \multirow{2}{*}{Model} & products of   & discrete              & continuous        & multiple\\
                                    & polynomials   & cond.variable         & cond. variable    & cond. variables\\
             \hline
             \modelnamePI~\citep{chrysos2020poly}       & \colorcheck & \xmark & \xmark & \xmark \\\hline
             StyleGAN~\citep{karras2018style}           & \xmark & \xmark & \xmark & \xmark \\\hline
             sBN~\citep{chen2018self}                   & \xmark & \colorcheck &  \xmark & \xmark\\\hline
             \modelnamespade~\citep{park2019semantic}   & \xmark & \xmark & \colorcheck &  \xmark\\\hline
             \modelnameMVP{} (ours)                     & \colorcheck & \colorcheck & \colorcheck & \colorcheck\\\hline
         \end{tabular}
         }
         \label{tab:mvp_table_comparison_with_other_polynomial_nets}
\end{table*}

 \section{Method}
\label{sec:mvp_method}

In the following paragraphs we introduce the two-variable polynomial expansion (sec.~\ref{ssec:mvp_method_two_variable}), while the detailed derivation, along with additional models are deferred to the \supplementary{} (sec.~\ref{sec:mvp_method_suppl}). The crucial technical details, including the stability of the polynomial, are developed in sec.~\ref{ssec:mvp_method_extensions}. We emphasize that a multivariate polynomial can approximate any function~\citep{stone1948generalized, nikol2013analysis}, i.e., a multivariate polynomial is a universal approximator.

\begin{wraptable}[12]{r}{0.55\textwidth}
    \vspace{-2.3mm}
    \caption{Symbols}
    \label{tbl:mvp_primary_symbols}
    \centering
    \begin{tabular}{|c | c|}
    \toprule
    Symbol 	& Role \\
    \midrule
    $N$ 		                    &	Expansion order of the polynomial\\
    $k$ 		                    & Rank of the decompositions \\
    $\sbinvar{I}, \sbinvar{II}$     & Inputs to the polynomial \\
    $n, \rho$                       & Auxiliary variables \\
    $\bmcal{W}^{[n, \rho]}$         & Parameter tensor of the polynomial\\
    $\bm{U}\matnot{n}, \bm{C}, \bm{\beta}$ 		    & Learnable parameters \\
    $*$                      & Hadamard product \\
     \hline
    \end{tabular}
    \vspace{-2pt}
\end{wraptable}

\textbf{Notation}:Tensors/matrices/vectors are symbolized by calligraphic/uppercase/lowercase boldface letters e.g., $\bmcal{W}$,$\bm{W}$,$\bm{w}$.
The \textit{mode-$m$ vector product} of $\bmcal{W}$ (of order $M$) with a
vector $\bm{u} \in \realnum^{I_m}$ is
$\bmcal{W} \times_{m} \bm{u}$ and results in a tensor of order $M-1$. 
We assume that $\prod_{i=a}^b x_i = 1$ when $a > b$.  The core symbols are summarized in Table~\ref{tbl:mvp_primary_symbols}, while a detailed tensor notation is deferred to the \supplementary{} (sec.~\ref{ssec:mvp_notation}).

\subsection{Two input variables}
\label{ssec:mvp_method_two_variable}
Given two input variables \footnote{To avoid cluttering the notation we use same dimensionality for the two inputs. However, the derivations apply for different dimensionalities, only the dimensionality of the tensors change slightly.} 
$\sbinvar{I}, \sbinvar{II} \in \knum^d$ where $\knum \subseteq \realnum$ or $\knum \subseteq \naturalnum$, the goal is to learn a function $G: \knum^{d\times d} \to \realnum^{o}$ that captures the higher-order correlations between the elements of the two inputs. We can learn such higher-order correlations as polynomials of two input variables. A polynomial expansion of order $N \in \naturalnum$ with output $\boutvar \in \realnum^o$ (such that $\boutvar = G(\sbinvar{I}, \sbinvar{II})$) has the form: 

\begin{equation}
    \boutvar = \sum_{n=1}^N \sum_{\rho=1}^{n+1} \bigg(\bmcal{W}^{[n, \rho]} \prod_{j=2}^{\rho} \times_{j} \sbinvar{I}  \prod_{\tau=\rho+1}^{n+1} \times_{\tau} \sbinvar{II}\bigg) + \bm{\beta}
    \label{eq:mvp_poly_general_eq_2var}
\end{equation}

where $\bm{\beta} \in \realnum^o$ and $\bmcal{W}^{[n, \rho]} \in  \realnum^{o\times \prod_{m=1}^{n}\times_m d}$ for $n\in [1, N], \rho\in [1, n+1]$ are the learnable parameters. The expansion depends on two (independent) variables, hence we use the $n$ and $\rho$ as auxiliary variables. The two products of (\ref{eq:mvp_poly_general_eq_2var}) do not overlap, i.e., the first multiplies the modes $[2, \rho]$ (of $\bmcal{W}^{[n, \rho]}$) with $\sbinvar{I}$ and the other multiplies the modes $[\rho+1, n+1]$ with $\sbinvar{II}$. 

\begin{figure*}[!h]
    \centering
    \includegraphics[width=0.9\linewidth]{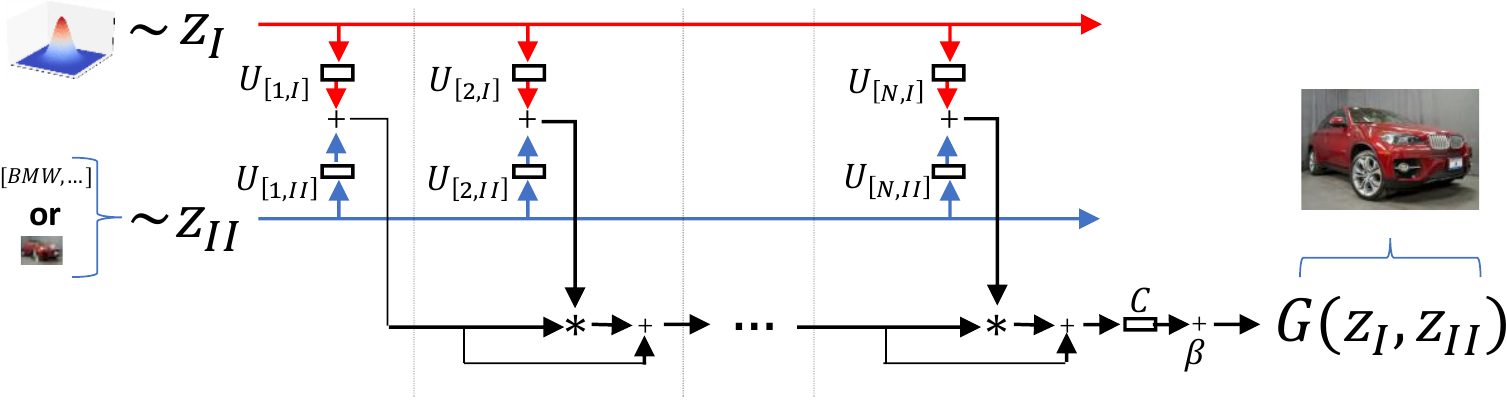}
\caption{Abstract schematic for $N^\text{th}$ order approximation of $\boutvar = G(\binvar_{I}, \binvar_{II})$. The inputs $\binvar_{I}, \binvar_{II}$ are symmetric in our formulation. We denote with $\binvar_{I}$ a noise vector, e.g., a samples from Gaussian distribution, while $\binvar_{II}$ symbolizes a sample from a conditional input (e.g., a class label or a low-resolution image). }
\label{fig:mvp_nth_order_schematic_2var}
\end{figure*}

\textbf{Recursive relationship}: The aforementioned derivation can be generalized to an arbitrary expansion order. The recursive formula for an arbitrary order $N \in \naturalnum$ is the following:

\begin{equation}
    \boutvar_{n} = \boutvar_{n-1} + \Big(\bm{U}\matnot{n, I}^T \sbinvar{I} +  \bm{U}\matnot{n, II}^T \sbinvar{II}\Big) * \boutvar_{n-1} 
    \label{eq:mvp_model1_rec}
\end{equation}
for $n=2,\ldots,N$ with $\boutvar_{1} = \bm{U}\matnot{1, I}^T \sbinvar{I} +  \bm{U}\matnot{1, II}^T \sbinvar{II}$ and $\boutvar = \bm{C}\boutvar_{N} + \bm{\beta}$. The parameters $\bm{C} \in \realnum^{o\times k}, \bm{U}\matnot{n, \phi} \in  \realnum^{d\times k}$ for $n=1,\ldots,N$ and $\phi=\{I, II\}$ are learnable. 

The intuition behind this model is the following: An embedding is initially found for each of the two input variables, then the two embeddings are added together and they are multiplied elementwise with the previous approximation. The different embeddings for each of the input variables allows us to implement $\bm{U}\matnot{n, I}$ and $\bm{U}\matnot{n, II}$ with different constraints, e.g., $\bm{U}\matnot{n, I}$ to be a dense layer and $\bm{U}\matnot{n, II}$ to be a convolution.

\begin{figure*}[!h]
\centering
    \centering
    \subfloat[Random samples per class]{\includegraphics[width=0.33\linewidth]{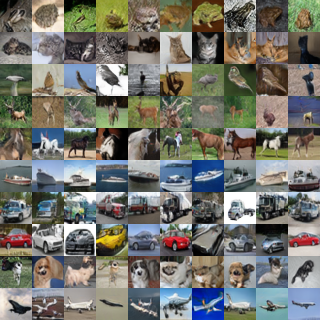}\hspace{1mm}}
    \subfloat[Intra-class interpolation]{\includegraphics[width=0.322\linewidth]{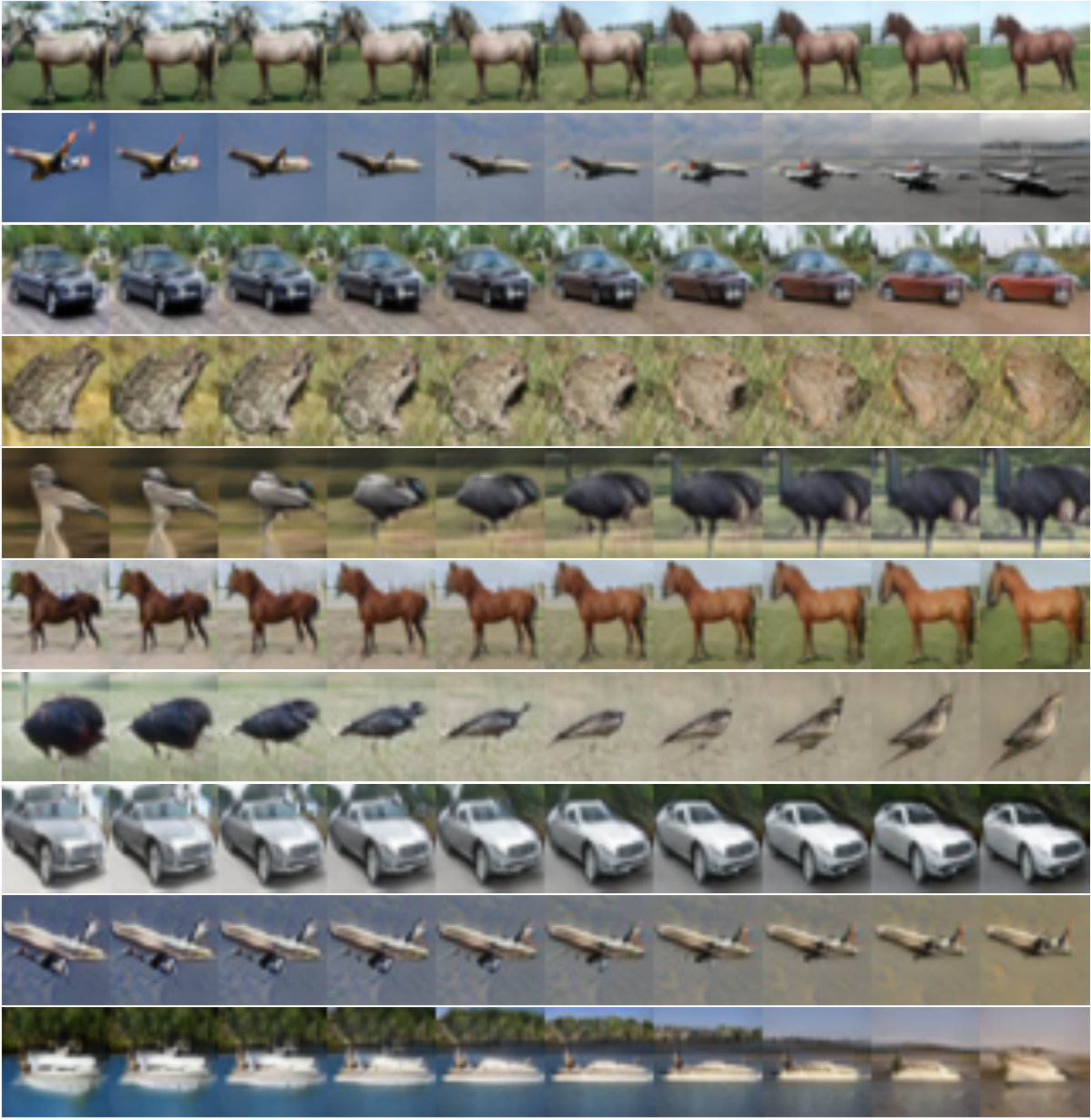}\hspace{1mm}}
    \subfloat[Inter-class interpolation]{\includegraphics[width=0.322\linewidth]{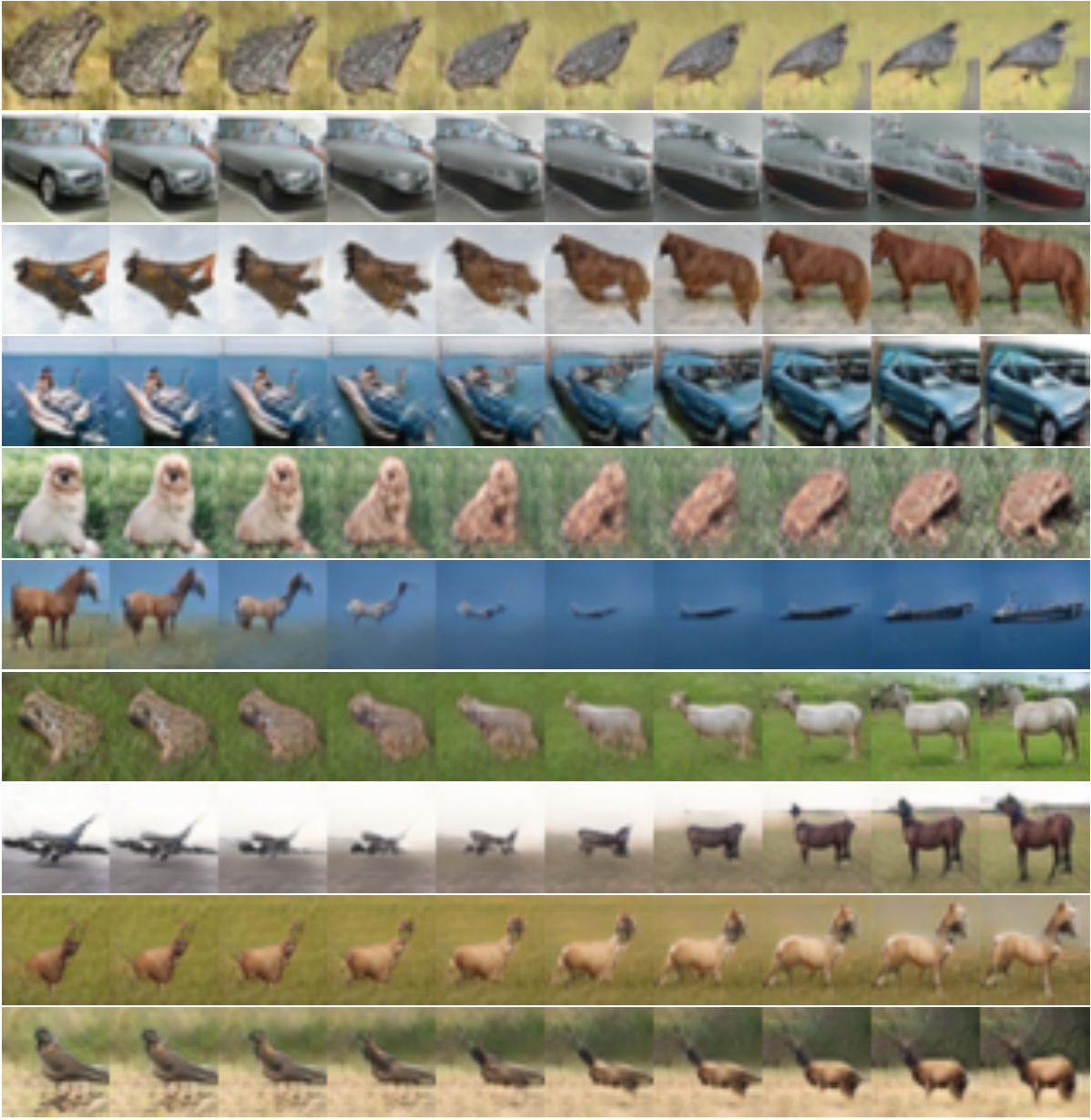}\hspace{1mm}}
\caption{Synthesized images by \modelnameMVP{} in the class-conditional CIFAR10 (with resnet-based generator): (a) Random samples where each row depicts the same class, (b) Intra-class linear interpolation from a source to the target, (c) inter-class linear interpolation. In inter-class interpolation, the class labels of the leftmost and rightmost images are one-hot vectors, while the rest are interpolated in-between; the resulting images are visualized. In all three cases, \modelnameMVP{} synthesizes realistic images.}
\label{fig:mvp_qualitative_class_cond}
\end{figure*}

\subsection{Model extensions and technical details} 
\label{ssec:mvp_method_extensions}

There are three  limitations in (\ref{eq:mvp_model1_rec}). Those are the following: a) (\ref{eq:mvp_model1_rec}) describes a polynomial expansion of a two-variable input, b) each expansion order requires additional layers, c) high-order polynomials might suffer from unbounded values. Those limitations are addressed below.

Our model can be readily extended beyond two-variable input; an extension with three-variable input is developed in sec.~\ref{sec:mvp_method_many_inputs_suppl}. The pattern (for each order) is similar to the two-variable input: a) a different embedding is found for each input variable, b) the embeddings are added together, c) the result is multiplied elementwise with the representation of the previous order.

The polynomial expansion of (\ref{eq:mvp_model1_rec}) requires $\Theta(N)$ layers for an $N^\text{th}$ order expansion. That is, each new order $n$ of expansion requires new parameters $\bm{U}\matnot{n, I}$ and $\bm{U}\matnot{n, II}$. However, the order of expansion can be increased without increasing the parameters substantially. To that end, we can capitalize on the product of polynomials.  Specifically, let $N_1$ be the order of expansion of the first polynomial. The output of the first polynomial is fed into a second polynomial, which has expansion order of $N_2$. Then, the output of the second polynomial will have an expansion order of $N_1\cdot N_2$. The second polynomial of degree $N_2$ can either be a polynomial of one variable (i.e., the output of the previous polynomial) or more variables (i.e., the output of the previous polynomial and one or more of the inputs). In both cases the total degree of the output of the second polynomial will be $N_1 \cdot N_2$. The choice of the type of polynomial in each case is a design option. The product of polynomials can be used with arbitrary number of polynomials; it suffices the output of the $\tau^\text{th}$ polynomial to be the input to the $(\tau+1)^\text{th}$ polynomial. For instance, if we assume a product of $\Phi \in \naturalnum$ polynomials, where each polynomial has an expansion order of two, then the polynomial expansion is of $2^{\Phi}$ order. In other words, we need $\Theta(\log_2(N))$ layers to achieve an $N^\text{th}$ order expansion.

In algebra, higher-order polynomials are unbounded and can thus suffer from instability for large values. To avoid such instability, we take the following three steps: a) \modelnameMVP{} samples the noise vector from the uniform distribution, i.e., from the bounded interval of $[-1, 1]$, b) a hyperbolic tangent is used in the output of the generator as a normalization, i.e., it constrains the outputs in the bounded interval of $[-1, 1]$, c) batch normalization~\citep{ioffe2015batch} is used to convert the representations to zero-mean. We emphasize that in GANs the hyperbolic tangent is the default activation function in the output of the generator, hence it is not an additional requirement of our method. Additionally, in our preliminary experiments, the uniform distribution can be changed for a Gaussian distribution without any instability. 
A theoretical analysis on the bounds of such multivariate polynomials would be an interesting subject for future work.  

Lastly, we highlight the flexibility of the proposed \modelnameMVP. 
sBN and \modelnamespade{} can be considered as special cases of the two-variable polynomial expansion. In particular, we exhibit in sec.~\ref{sec:mvp_difference_from_spade_suppl} how \modelnamespade{} can be extended into a general-purpose two-variable polynomial expansion. In addition, the products of polynomials would enable both sBN and \modelnamespade{} to perform higher-order expansions without increasing the number of layers.

 \vspace{-9pt}
\section{Experiments}
\label{sec:mvp_experiments}
\vspace{-6pt}

To validate the proposed formulation, the following diverse conditional generation tasks are considered: 
\begin{itemize} 
    \vspace{-1.1em}
    \setlength\itemsep{-0.3em}
    \item class-conditional generation trained on CIFAR10, Cars196 and SVHN in sec.~\ref{ssec:mvp_experiments_discrete} and sec.~\ref{ssec:mvp_experiment_class_conditional_svhn_suppl}.
    \item generation of unseen attribute combinations in sec.~\ref{ssec:mvp_experiments_unseen_combinations}. 
    \item attribute-guided generation in sec.~\ref{ssec:mvp_experiment_multiple_discrete_conditional_input_suppl}. 
    \item inverse problems in imaging, e.g., super-resolution and block-inpainting, trained on Cars196 and CelebA in sec.~\ref{ssec:mvp_experiments_continuous}. 
    \item edges-to-image translation trained on handbags and shoes in sec.~\ref{ssec:mvp_experiment_translation_edge2im_suppl}.
    \item image-to-image translation in sec.~\ref{ssec:mvp_experiment_translation_mnist_svhn_suppl}. 
\end{itemize}
\vspace{-2pt}

The details on the datasets and the evaluation metrics are deferred to the \supplementary{} (sec.~\ref{sec:mvp_experimental_details_suppl}) along with additional visualizations and experiments. Additionally, the source code of \modelnameMVP{} is available at \url{https://github.com/grigorisg9gr/polynomial_nets_for_conditional_generation}.

\begin{figure*}
\centering
    \centering
    \subfloat[Class-conditional generation]{\includegraphics[width=0.32\linewidth]{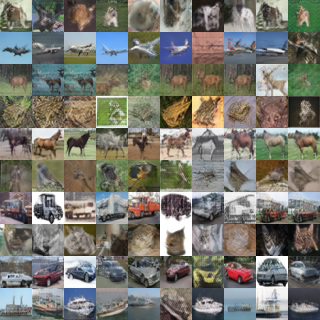}\hspace{1mm}}
    \subfloat[Class-conditional generation]{\includegraphics[width=0.32\linewidth]{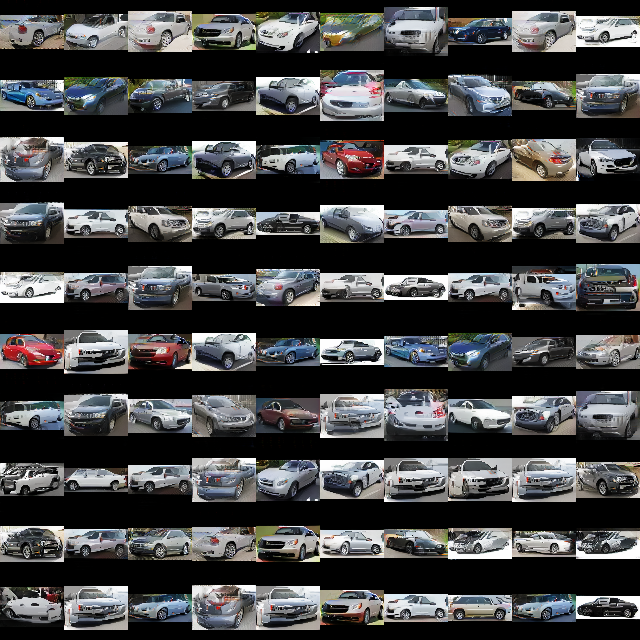}\hspace{1mm}}
    \subfloat[Block-inpainting]{\includegraphics[width=0.2675\linewidth]{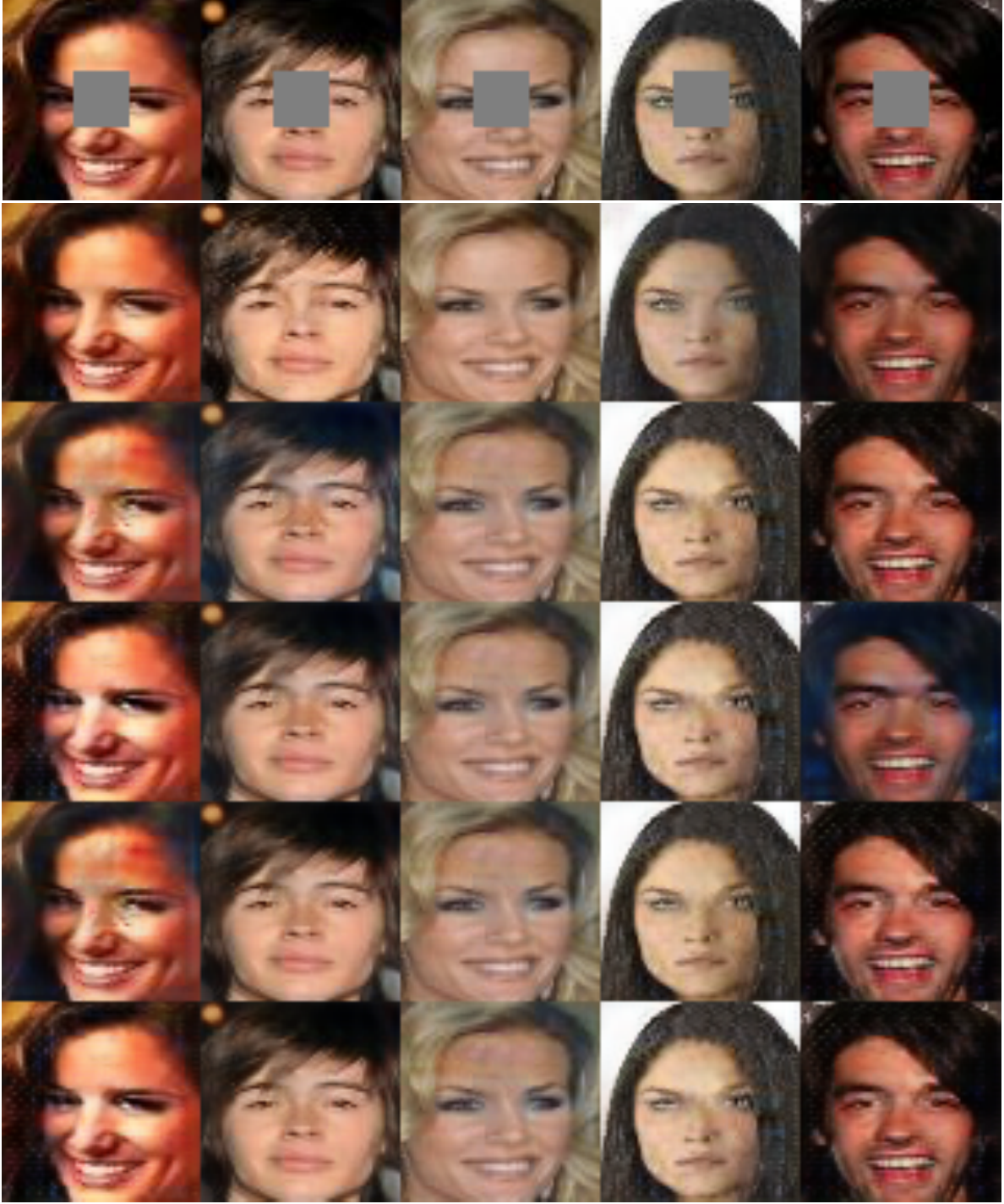}}
\caption{Synthesized images by \modelnameMVP{} in the (a), (b) class-conditional generation (sec.~\ref{ssec:mvp_experiments_discrete}) and (c) block-inpainting (sec.~\ref{ssec:mvp_experiments_continuous}). In class-conditional generation, each row depicts a single class.}
\label{fig:mvp_qualitative_class_cond_linear}
\end{figure*}

Our framework, e.g., (\ref{eq:mvp_model1_rec}), does not include any activation functions. To verify the expressivity of our framework, we maintain the same setting for the majority of the experiments below. Particularly, the generator does not have activation functions between the layers; there is only a hyperbolic tangent in the output space for normalization as typically done in GAN generators. 
However, we conduct one experiment using a strong baseline with activation functions. That is, a comparison with SNGAN~\citep{miyato2018cgans} in class-conditional generation is performed (sec.~\ref{ssec:mvp_experiments_discrete}). 

\paragraph{Baselines:} `\modelnamePI-SICONC' implements a polynomial expansion of a single variable by concatenating all the input variables. `\modelnamespade' implements a polynomial expansion with respect to the conditional variable. Also, `GAN-CONC' and `GAN-ADD' are added as baselines, where we replace the Hadamard products with concatenation and addition respectively. A schematic of the differences between the compared polynomial methods is depicted in Fig.~\ref{fig:mvp_schematics_product_polynomials}, while a detailed description of all methods is deferred to sec.~\ref{sec:mvp_experimental_details_suppl}. 
Each experiment is conducted \textbf{five} times and the mean and the standard deviation are reported. Throughout the experimental section, we reserve the symbol $\sbinvar{II}$ for the conditional input (e.g., a class label).

\vspace{-3pt}
\subsection{Class-conditional generation}
\label{ssec:mvp_experiments_discrete}
\vspace{-3pt}
In class-conditional generation the conditional input is a class label in the form of one-hot vector. The experiments we conduct below modify only the generator, while in all cases we assume there is the same discriminator. In particular, two types of generators are used: a) a resnet-based generator (SNGAN), b) a polynomial generator (\modelnamePI). The former network has exhibited strong performance the last few years, while the latter is a recently proposed PNN.

\begin{table}[tb]
\caption{Quantitative evaluation on class-conditional generation with resnet-based generator (i.e., SNGAN). Higher Inception Score (IS)~\citep{salimans2016improved} (lower Frechet Inception Distance (FID)~\citep{heusel2017gans}) indicates better performance. The baselines improve the IS of SNGAN, however they cannot improve the FID. Nevertheless, \modelnameMVPsngan{} improves upon all the baselines in both the IS and the FID.}
\centering
\begin{tabular}{ll}
  \begin{minipage}{.495\textwidth}
         \vspace{2pt}
         \begin{tabular}{|c | c | c|} 
             \hline
             \multicolumn{3}{|c|}{class-conditional generation on CIFAR10}\\ 
             \hline
             Model & IS ($\uparrow$) & FID ($\downarrow$)\\
             \hline
             SNGAN & $8.30\pm 0.11$ & $14.70\pm 0.97$\\\hline
             SNGAN-CONC & $8.50\pm 0.49$ & $30.65\pm 3.55$\\\hline
             SNGAN-ADD & $8.65\pm 0.11$ & $15.47\pm 0.74$\\\hline
             SNGAN-\modelnamespade & $8.69\pm 0.19$ & $21.74\pm	0.73$\\\hline
             \modelnameMVPsngan & $\bm{8.77 \pm 0.12}$ & $\bm{14.22 \pm 0.66}$\\
             \hline
         \end{tabular}
         \label{tab:mvp_exper_cifar10_sota}
  \end{minipage}
\end{tabular}
\end{table}

\paragraph{Resnet-based generator:} The experiment is conducted by augmenting the resnet-based generator of SNGAN\footnote{\label{footnote:experiments_generator}Further implementation details are offered in sec.~\ref{sec:mvp_experimental_details_suppl}.}. Each compared method will be named according to the modification on the generator of SNGAN, e.g., \modelnameMVPsngan{} utilizes a resnet-based generator following the proposed framework of (\ref{eq:mvp_model2_rec}). All methods are trained using CIFAR10 images; CIFAR10 is a popular benchmark in class-conditional generation. The quantitative results are in Table~\ref{tab:mvp_exper_cifar10_sota} and synthesized samples are illustrated in Fig.~\ref{fig:mvp_qualitative_class_cond}(a).  \modelnameMVPsngan{} improves upon all the baselines in both the Inception score (IS)~\citep{salimans2016improved} and the FID~\citep{heusel2017gans}. The proposed formulation enables inter-class interpolations. That is, the noise $\sbinvar{I}$ is fixed, while the class $\sbinvar{II}$ is interpolated. 
In Fig.~\ref{fig:mvp_qualitative_class_cond}(b) and Fig.~\ref{fig:mvp_qualitative_class_cond}(c), intra-class and inter-class linear interpolations are illustrated respectively. Both the quantitative and the qualitative results exhibit the effectiveness of our framework.

\begin{table*}[tb]
\caption{Quantitative evaluation on class-conditional generation with \modelnamePI-based generator. In CIFAR10, there is a considerable improvement on the IS, while in Cars196 FID drops dramatically with \modelnameMVP. We hypothesize that the dramatic improvement in Cars196 arises because of the correlations of the classes. For instance, the SUV cars (of different carmakers) share several patterns, which are captured by our high-order interactions, while they might be missed when learning different normalization statistics per class. The generator \emph{does not} have activation functions between the layers, so the deteriorated performance of GAN-CONC and GAN-ADD is reasonable.}

\centering
\begin{tabular}{ll}

  \begin{minipage}{.53\textwidth}
    \centering \hspace{-8mm}
         \vspace{0pt}
         \begin{tabular}{|c | c | c|} 
             \hline
             \multicolumn{3}{|c|}{class-conditional generation on CIFAR10}\\ 
             \hline
             Model & IS ($\uparrow$) & FID ($\downarrow$)\\
             \hline
             GAN-CONC & $3.73 \pm 0.32$ & $294.33\pm 8.16$\\\hline
             GAN-ADD & $3.74 \pm 0.60$ & $298.53\pm 16.54$\\\hline
             \modelnamespade & $4.00  \pm	0.53$ &	 $294.21  \pm	16.33$ \\ \hline
             \modelnamePI-SICONC & $6.65 \pm 0.60$ & $71.81 \pm 33.00$\\\hline 
             \modelnamePI & $7.54 \pm 0.16$ & $37.26 \pm 1.86$\\\hline 
             \modelnameMVP & $\bm{7.87 \pm 0.21}$ & $\bm{34.35 \pm 2.68}$\\
             \hline
         \end{tabular}
         \label{tab:mvp_exper_cifar10_linear}
  \end{minipage}
  \hfill  
  
  \begin{minipage}{.38\textwidth}
         \vspace{0pt}\hspace{-2.5mm}
         \begin{tabular}{|c | c |} 
             \hline
             \multicolumn{2}{|c|}{class-conditional generation on Cars196}\\ 
             \hline
             Model & FID ($\downarrow$)\\
             \hline
             GAN-CONC   & $240.45  \pm	16.79$\\\hline
             GAN-ADD    & $208.72  \pm	12.65$\\\hline
             \modelnamespade & $168.19  \pm	39.71$\\ \hline   
             \modelnamePI-SICONC & $153.39  \pm	27.93$\\\hline 
             \modelnamePI & $120.40 \pm	28.65$\\\hline
             \modelnameMVP & $\bm{55.48 \pm 3.16}$\\
             \hline
         \end{tabular}
         \label{tab:mvp_exper_cars196_linear_class_cond}
  \end{minipage}
\end{tabular}
\end{table*}

\paragraph{\modelnamePI-based generator:} A polynomial expansion is selected as the baseline architecture for the generator\textsuperscript{\ref{footnote:experiments_generator}}. In the original \modelnamePI{} conditional batch normalization (CBN) was used in the generator; this is replaced by batch normalization in the rest of the compared methods. 
The quantitative results in CIFAR10 are summarized in Table~\ref{tab:mvp_exper_cifar10_linear} (left). \modelnamespade{} does not utilize the products of polynomials formulation, which explains its poor performance. Additionally, even though \modelnamePI-SINCONC and \modelnamePI{} both express a single-variable polynomial expansion, the inductive bias inserted into the network has a substantial effect in the final performance. Notice that \modelnameMVP{} outperforms all the baselines by a large margin. 

We also evaluate class-conditional generation in Cars196 that has $196$ classes. Cars196 is selected as a reasonably larger dataset than CIFAR10; it contains $196$ classes and yet it can be trained on a single GPU. The compared methods and the training details remain the same. The results in Table~\ref{tab:mvp_exper_cifar10_linear} (right) demonstrate a substantial difference between \modelnameMVP{} and the compared methods. Namely, the proposed method achieves a $53.9\%$ reduction of the FID over the best-performing baseline. 
We emphasize that both \modelnamespade{} and \modelnamePI{} were not originally built for class-conditional generation, however we have tried to optimize the respective hyper-parameters to optimize their performance. The performance gap between \modelnamespade{} and the rest PNNs can be explained by the lack of products of polynomials. We also verify the experimental results on CIFAR10 that demonstrate how \modelnamePI{} improves upon \modelnamePI-SINCONC. However, even \modelnamePI{} obtains a substantially higher FID than \modelnameMVP; we hypothesize that the improvement arises because of the correlations between the classes. For instance, the SUV cars of different carmakers share several patterns. Such correlations are captured by our framework, while they might be missed when learning different normalization statistics per class. 
Overall, \modelnameMVP{} synthesizes plausible images (Fig.~\ref{fig:mvp_qualitative_class_cond_linear}) even in the absence of activation functionsn.

\vspace{-3pt}
\subsection{Polynomial conditioning for generating unseen attribute combinations}
\label{ssec:mvp_experiments_unseen_combinations}
\vspace{-3pt}

The proposed \modelnameMVP{} is a general framework for conditional generation, and we have already demonstrated how it can be used for conditional generation in GANs. In this section, we extend our method to the Variational Autoencoder (VAE)~\citep{kingma2013auto} to showcase a useful byproduct of our formulation, namely the ability to generate unseen label combinations.
In particular, as introduced in \citep{georgopoulos2020multilinear}, this is a multi-label setting where one (or more) combinations are not seen in the training set. The method is then evaluated based on its ability to generate the unseen attribute combinations. To this end, we implement \modelnameMVP-VAE, a variation of the conditional VAE~\citep{kingma2013auto} where both the encoder and decoder are polynomials. We perform experiments on CelebA using the annotated attributes of smile and gender. Similar to \citep{georgopoulos2020multilinear} we remove the combination (Smiling, Female) from the training set. The results in Figure \ref{fig:cope_vae_unseen_combinations} highlight the efficacy of the proposed conditioning method in disentangling the two labels and leveraging the multiplicative interactions to synthesize the missing combination.
\vspace{-3pt}

\begin{figure*}
\centering
    \centering
    \subfloat{\includegraphics[width=0.8\linewidth]{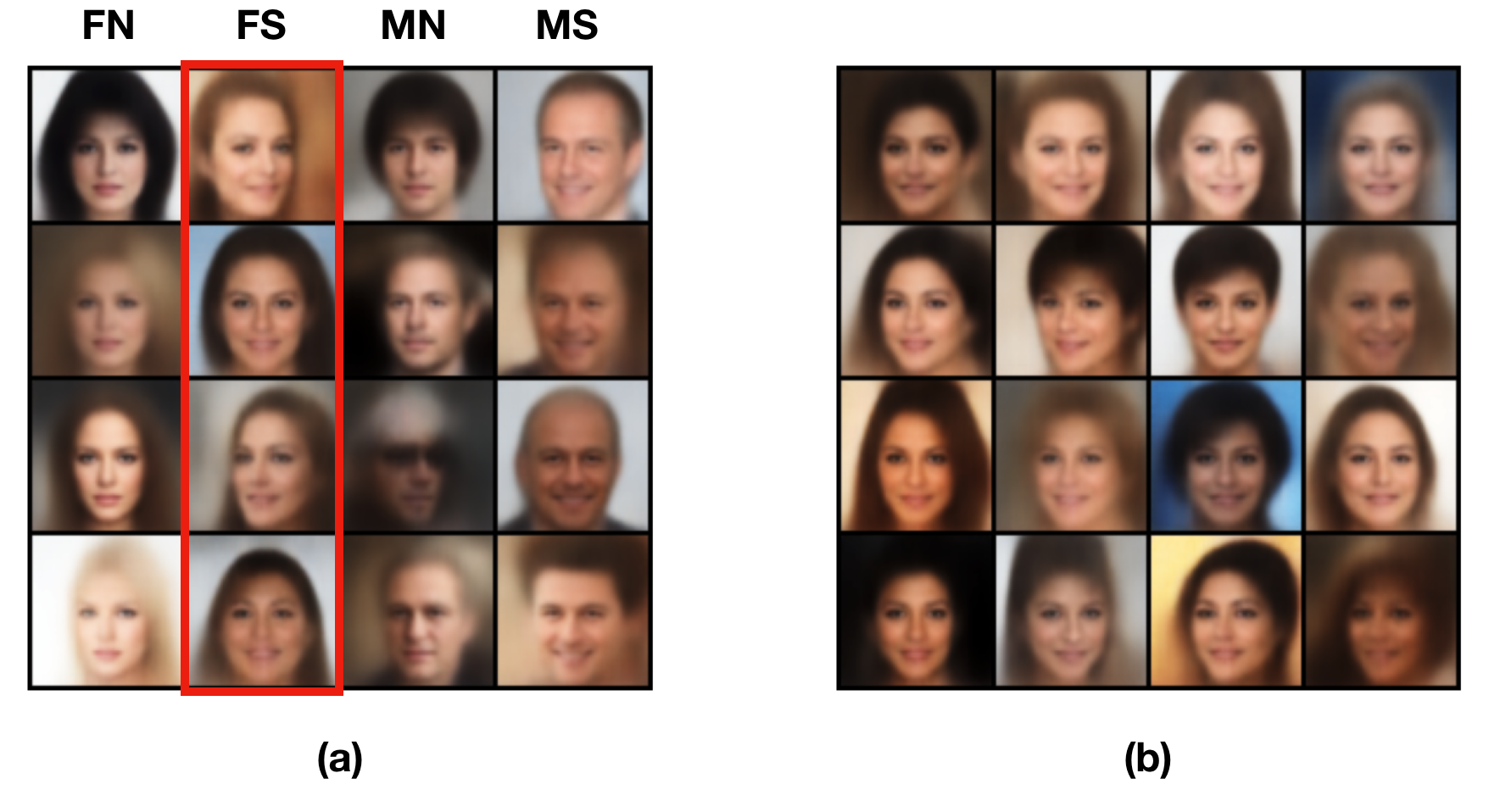}\hspace{1mm}}
\caption{Synthesized images with \modelnameMVP-VAE. In (a), all combinations are illustrated (the red is the combination missing during training, i.e. Female+Smile), while in (b), only images from the missing combination are visualized.}
\label{fig:cope_vae_unseen_combinations}
\end{figure*}
\vspace{-3pt}

\subsection{Inverse problems in imaging}
\label{ssec:mvp_experiments_continuous}
\vspace{-3pt}
In the following paragraphs we evaluate the performance of \modelnameMVP{} in inverse problems. We select super-resolution and block-inpainting as two popular tasks. 

The core architectures remain as in the experiment above, i.e., \modelnameMVP{} and \modelnamePI-SICONC implement products of polynomials. A single change is made in the structure of the discriminator: Motivated by \citep{miyato2018cgans}, we include an elementwise product of $\sbinvar{II}$ with the real/fake image in the discriminator. This stabilizes the training and improves the results. Even though architectures specialized for a single task (e.g., \citet{ledig2016photo}) perform well in that task, their well-selected inductive biases (e.g., perceptual or $\ell_1$ loss) do not generalize well in other domains or different conditional inputs. Our goal is not to demonstrate state-of-the-art results, but rather to scrutinize the effectiveness of the proposed formulation in different conditional generation tasks. To that end, we consider \modelnamePI-SICONC, \modelnamespade{} and \modelnamespade-\modelnameMVP{} as the baselines.

\begin{table*}[tb]
\caption{Quantitative evaluation on super-resolution with \modelnamePI-based generator on Cars196. The task on the left is super-resolution $16\times$, while on the right the task is super-resolution $8\times$. Our variant of \modelnamespade, i.e., \modelnamespade-\modelnameMVP{} (details in sec.~\ref{sec:mvp_experimental_details_suppl}),  vastly improves the original \modelnamespade. The full two-variable model, i.e., \modelnameMVP, outperforms the compared methods.}
\centering
\vspace{0pt}
\begin{tabular}{ll}

  \begin{minipage}{.39\textwidth}
  \scalebox{0.94}{
         \vspace{0pt}\hspace{-3mm}
         \begin{tabular}{|c | c |} 
             \hline
             \multicolumn{2}{|c|}{Super-resolution $16\times$ Cars196}\\ \hline  
             Model & FID ($\downarrow$)\\\hline
             \modelnamespade & $111.75  \pm	13.41$ \\ \hline
             \modelnamePI-SICONC & $80.16  \pm	12.42$\\\hline 
             \modelnamespade-\modelnameMVP & $72.63  \pm	3.18$\\\hline
             \modelnameMVP & $\bm{60.42  \pm	6.19}$\\\hline
         \end{tabular}
         }
         \label{tab:mvp_exper_superres_4px_cars196_linear}
  \end{minipage}
  \hfill  
  
  \begin{minipage}{.60\textwidth}
         \vspace{0pt}\hspace{-6mm}
        \scalebox{0.94}{
         \begin{tabular}{|c | c | c | c|} 
             \hline
             \multicolumn{4}{|c|}{Super-resolution $8\times$ Cars196}\\ \hline
             Model & FID ($\downarrow$) & SSIM ($\uparrow$) & LPIPS ($\downarrow$)\\\hline
             \modelnamespade & $119.18  \pm	14.82$ & $0.32$ & $0.178$\\ \hline
             \modelnamePI-SICONC & $186.42  \pm	40.84$ & $0.31$ & $0.200$ \\\hline 
             \modelnamespade-\modelnameMVP & $64.76  \pm	8.26$ &  $0.49$ & $0.135$ \\\hline 
             \modelnameMVP & $\bm{62.76  \pm	4.37}$ & $\bm{0.53}$ & $\bm{0.127}$ 	\\\hline
         \end{tabular}
         }
         \label{tab:mvp_exper_superres_8px_cars196_linear}
  \end{minipage}
\end{tabular}
\end{table*}

We experiment with two settings in super-resolution: one that the input image is down-sampled $8\times$ and one that it is down-sampled $16\times$. The two settings enable us to test the granularity of \modelnameMVP{} at different scales. In super-resolution $16\times$, $\sbinvar{II}$ (i.e., the low-resolution input) has $48$ dimensions, while in super-resolution $8\times$, $\sbinvar{II}$ has $192$ dimensions. 
The FID scores in Cars196 for the task of super-resolution are reported in Table~\ref{tab:mvp_exper_superres_8px_cars196_linear}. In addition, for the experiment on super-resolution $8\times$, the SSIM~\citep{wang2004image} and LPIPS~\citep{zhang2018unreasonable} are reported as widely-used metrics in inverse imaging tasks. Notice that the performance of \modelnamePI-SICONC deteriorates substantially when the dimensionality of the conditional variable increases. That validates our intuition about the concatenation in the input of the generator (sec.~\ref{sec:mvp_difference_from_pinet_suppl}), i.e., that the inductive bias of single-variable PNNs might not fare well in conditional generation tasks. 
We also report the \modelnamespade-\modelnameMVP, which captures higher-order correlations with respect to the first variable as well (further details in sec.~\ref{sec:mvp_experimental_details_suppl}). The proposed \modelnamespade-\modelnameMVP{} outperforms the original \modelnamespade, however it cannot outperform the full two-variable model, i.e., \modelnameMVP. The results indicate that \modelnameMVP{} performs well even when the conditional input is an image.

Beyond the quantitative results, qualitative results provide a different perspective on what the mappings learn. Qualitative results on the super-resolution experiments on Cars196 are provided in Fig.~\ref{fig:mvp_qualitative_continuous_cond}. We also provide synthesized results on both super-resolution $8\times$ and block-inpainting on CelebA in Fig.~\ref{fig:mvp_qualitative_continuous_cond} and Fig.~\ref{fig:mvp_qualitative_class_cond_linear} respectively. For each conditional image, different noise vectors $\sbinvar{I}$ are sampled. Notice that the corresponding synthesized images differ in the fine details. For instance, changes in the mouth region, the car type or position and even background changes are observed. Thus, \modelnameMVP{} synthesizes realistic images that i) correspond to the conditional input, ii) vary in the fine details. Similar variation has emerged even when the source and the target domains differ substantially, e.g., in the translation of MNIST digits to SVHN digits (sec.~\ref{ssec:mvp_experiment_translation_mnist_svhn_suppl}). We should mention that the aforementioned experiments were conducted only using the adversarial learning loss. In the literature, regularization techniques have been proposed specifically for image-to-image translation, e.g., \citet{yang2019diversity, lee2019harmonizing}. However, such works utilize additional losses and even require additional networks for training, which makes the training computationally demanding and more sensitive to design choices.

 \vspace{-2pt}
\section{Discussion}
\label{sec:mvp_discussion}
\vspace{-4pt}

\modelnameMVP{} can be used for various conditional generation tasks as the experimental evaluation in both sec.~\ref{sec:mvp_experiments} and the \supplementary{} illustrate. Namely, \modelnameMVP{} can synthesize diverse content, which is typically tackled using auxiliary losses or networks in conditional GANs. We expect this attribute of diverse generation to be useful in inverse tasks, where multiple latent sharp images can correspond to a single corrupted image. The diverse generation can be attributed to the higher-order correlations between the noise and the conditional variable. Such higher-order correlations also enable synthesizing images with unseen attribute combinations (sec.~\ref{ssec:mvp_experiments_unseen_combinations}). 

One limitation of our work is that our method has not been tried in large-scale synthesis, e.g., like \citet{brock2019large}. However, only a very limited number of labs/institutions have access to such resources. In addition, our single-GPU training does have a reduced energy footprint when compared to the multi-GPU setups of large scale GANs. We believe that the proposed method has merit despite the single-GPU training. Indeed, we demonstrate that \modelnameMVP{} can perform well in different tasks, which could help reduce the search for independent methods for every single task. 

An interesting future step would be to evaluate the performance of the proposed \modelnameMVP{} in training with limited data, e.g., using techniques similar to \citet{karras2020training}. This task can also encourage architecture discovery, similar to the various architectures devised in \citet{chrysos2020poly}. In this work, we have demonstrated two such architectures with recursive formulations as in (\ref{eq:mvp_model1_rec}) and (\ref{eq:mvp_model2_rec}), however additional architectures can be designed by changing the tensor decomposition. 

The equations of (\ref{eq:mvp_model1_rec}) and (\ref{eq:mvp_model2_rec}) express polynomial expansions of arbitrary order, which can approximate the target function without using activation functions between the recursive terms. However, in sec.~\ref{ssec:mvp_experiments_discrete} we have also experimented with generators that include activation functions between recursive terms (i.e., the experiment with SNGAN variants). In the future, we intend to study how to include the activation functions in the formulation and study the properties of such piecewise polynomial expansions. Lastly, demystifying the relationship between the order of the polynomial expansion and the expressivity (e.g., the implementation details for the order in sec.~\ref{sec:mvp_experimental_details_suppl}) is a promising direction. 

\textbf{Societal impact of image generation}: Manipulation of images is made possible through algorithms and architectures like ours with well-studied potential negative applications. The rapid progress in GAN-based image synthesis has made the discussion imperative. We encourage further work to be conducted on understanding how to detect synthesized images, e.g., in the conditional generation setting. For instance, our method can be used for training powerful classifiers that detect synthesized images.

 \vspace{-4pt}
\section{Conclusion}
\vspace{-6pt}

We have introduced \modelnameMVP{} for conditional data generation. \modelnameMVP{} expresses a polynomial expansion of two input variables,  i.e., a noise vector and a conditional variable. We exhibit how previously published methods, such as \modelnamespade{} and sBN, can be considered as special forms of this two-variable polynomial expansion. Notably, \modelnameMVP{} can be augmented to accept an arbitrary number of conditional variables as inputs. 
The empirical evaluation confirms that our framework can synthesize realistic images in five diverse tasks, including inverse problems and class-conditional generation. 
Inverse problems, such as super-resolution, can benefit from the proposed framework; we showcase that sampling different noise vectors results in plausible differences in the synthesized image. We derive two recursive formulations, i.e., \eqref{eq:mvp_model1_rec} and \eqref{eq:mvp_model2_rec}, but a new task-specific formulation can be easily defined. We expect this to be useful in learning from different modalities, such as visual question answering (VQA) or text-to-speech synthesis, since \modelnameMVP{} can capture high-order auto- and cross-correlations among the input variables.

 \section*{Acknowledgements}
\label{sec:mvp_acks}

This project was sponsored by the Department of the Navy, Office of Naval Research(ONR) under a grant number N62909-17-1-2111. We are thankful to the reviewers for their constructive feedback and recommendations on how to improve our manuscript.

\bibliography{egbib}
\bibliographystyle{abbrvnat}

\clearpage\newpage
\appendix
\section{Summary of sections in the Appendix}
\label{sec:mvp_intro_suppl}

In the following sections, further details and derivations are provided to elaborate the details of the \modelnameMVP. Specifically, in sec.~\ref{sec:mvp_method_suppl} the decomposition and related details on the method are developed. The extension of our method beyond two-input variables is studied in sec.~\ref{sec:mvp_method_many_inputs_suppl}. 
A method frequently used in the literature for fusing information is concatenation; we analyze how concatenation captures only additive and not more complex correlations (e.g., multiplicative) in sec.~\ref{sec:mvp_concatenation_suppl}. 
The differences from \modelnamePI~\citep{chrysos2020poly} is explored in sec.~\ref{sec:mvp_difference_from_pinet_suppl}. In sec.~\ref{sec:mvp_difference_from_other_polynomial_networks_suppl}, some recent (conditional) data generation methods are cast into the polynomial neural network framework and their differences from the proposed framework are analyzed.
The experimental details including the evaluation metrics and details on the baselines are developed in sec.~\ref{sec:mvp_experimental_details_suppl}. 
In sec.~\ref{sec:mvp_experiments_suppl}, additional experimental results are included. Lastly, the differences from works that perform diverse generation are explored in sec.~\ref{sec:mvp_difference_from_other_diverse_generation_techniques_suppl}.    \section{Method derivations}
\label{sec:mvp_method_suppl}

In this section, we expand on the method details, including the scalar output case or the notation. Specifically, a more detailed notation is determined in sec.~\ref{ssec:mvp_notation}; the scalar output case is analyzed in sec.~\ref{ssec:mvp_method_scalar_output_suppl}. 
In sec.~\ref{ssec:mvp_method_second_order_proof_ccp} a second order expansion is assumed to illustrate the connection between the polynomial expansion and the recursive formula. 
Sequentially, we derive an \emph{alternative} model with different factor sharing. This model, called \modeltwo, has a nested factor sharing format (sec.~\ref{ssec:mvp_method_second_order_proof_ncp}).

\subsection{Notation}
\label{ssec:mvp_notation}

Our derivations rely on tensors (i.e., multidimensional equivalent of matrices) and (tensor) products. We relay below the core notation used in our work, the interested reader can find further information in the tensor-related literature~\citep{kolda2009tensor, debals2017concept}. 

\textbf{Symbols of variables}: Tensors/matrices/vectors are symbolized by calligraphic/uppercase/lowercase boldface letters e.g., $\bmcal{W}$,$\bm{W}$,$\bm{w}$. 

\textbf{Matrix products}: The \textit{Hadamard} product of $\bm{A}, \bm{B} \in \realnum^{I \times N}$ is defined as $\bm{A} * \bm{B}$ and is equal to ${a}_{(i, j)} {b}_{(i, j)}$ for the $(i, j)$ element. The \textit{Khatri-Rao} product of matrices $\bm{A} \in \realnum^{I \times N}$
and $\bm{B} \in \realnum^{J \times N}$ is
denoted by $\bm{A} \odot \bm{B}$ and yields a matrix of
dimensions $(IJ)\times N$.  The Khatri-Rao product for a set of matrices  $\{\bm{A}\matnot{m} \in \realnum^{I_m \times N} \}_{m=1}^M$ is abbreviated by $\bm{A}\matnot{1} \odot \bm{A}\matnot{2} \odot  \cdots \odot  \bm{A}\matnot{M} \doteq  \bigodot_{m=1}^M \bm{A}\matnot{m}$.

\textbf{Tensors}: Each element of an $M^{th}$ order tensor $\bmcal{W}$ is addressed by $M$ indices, i.e., $(\bmcal{W})_{i_{1}, i_{2}, \ldots, i_{M}} \doteq w_{i_{1}, i_{2}, \ldots, i_{M}}$. An $M^{th}$-order tensor $\bmcal{W}$ is  defined over the
tensor space $\realnum^{I_{1} \times I_{2} \times \cdots \times
I_{M}}$, where $I_{m} \in \mathbb{Z}$ for $m=1,2,\ldots,M$. 
 The \textit{mode-$m$ unfolding} of a tensor $\bmcal{W} \in
 \realnum^{I_1 \times I_2 \times \cdots \times I_M}$ maps
 $\bmcal{W}$ to a matrix $\bm{W}_{(m)} \in \realnum^{I_{m}
 \times \bar{I}_{m}}$ with $\bar{I}_{m}= \prod_{k=1 \atop k  \neq m}^M I_k $ such
 that the tensor element $w_{i_1, i_2, \ldots, i_M}$ is
 mapped to the matrix element $w_{i_{m}, j}$ where
 $j=1 + \sum_{k=1 \atop k \neq m}^M (i_k - 1) J_k$ with $J_k =
\prod_{n =1 \atop n \neq m}^{k-1} I_n $. 
The \textit{mode-$m$ vector product} of $\bmcal{W}$ with a
vector $\bm{u} \in \realnum^{I_m}$, denoted by
$\bmcal{W} \times_{m} \bm{u} \in \realnum^{I_{1}\times
I_{2}\times\cdots\times I_{m-1}  \times I_{m+1} \times
\cdots \times I_{M}} $, results in a tensor of order $M-1$:
\begin{equation}\label{E:Tensor_Mode_n}
(\bmcal{W} \times_{m} \bm{u})_{i_1, \ldots, i_{m-1}, i_{m+1},
\ldots, i_{M}} = \sum_{i_m=1}^{I_m} w_{i_1, i_2, \ldots, i_{M}} u_{i_m}.
\end{equation}
We denote
$\bmcal{W} \times_{1} \bm{u}^{(1)} \times_{2} \bm{u}^{(2)} \times_{3}  \cdots \times_{M} \bm{u}^{(M)}  \doteq 
\bmcal{W} \prod_{m=1}^m \times_{m} \bm{u}^{(m)}$.

The \textit{CP decomposition}~\citep{kolda2009tensor} factorizes a tensor into a sum of component rank-one tensors. The rank-$R$ CP decomposition of an $M^{th}$-order tensor $\bmcal{W}$ is written as:
 \begin{equation}\label{E:CP}
\bmcal{W}  \doteq [\![ \bm{U}\matnot{1}, \bm{U}\matnot{2}, \ldots, \bm{U}\matnot{M}  ]\!] =  \sum_{r=1}^R \bm{u}_r^{(1)}  \circ \bm{u}_r^{(2)}  \circ \cdots \circ \bm{u}_r^{(M)},
\end{equation}
where $\circ$ is the vector outer product. The factor matrices $\big\{ \bm{U}\matnot{m} = [\bm{u}_1^{(m)},\bm{u}_2^{(m)}, \cdots, \bm{u}_R^{(m)} ] \in \mathbb{R}^{I_m \times R} \big\}_{m=1}^{M}$ collect 
the vectors from the rank-one components. By considering the mode-$1$ unfolding of $\bmcal{W}$, the CP decomposition can be written in matrix form as: 

 \begin{equation}
 \label{eq:polygan_cp_unfolding}
\bm{W}_{(1)}  
\doteq \bm{U}\matnot{1} \bigg( \bigodot_{m = M}^{2} \bm{U}\matnot{m}\bigg)^T
\end{equation}

The following lemma is useful in our method: 

\begin{lemma} 
For a set of $N$ matrices $\{\bm{A}\matnot{\nu} \in \realnum^{I_{\nu} \times K} \}_{\nu=1}^N$  and $\{\bm{B}\matnot{\nu} \in \realnum^{I_{\nu} \times L} \}_{\nu=1}^N$, the following equality holds:
\begin{equation}
    (\bigodot_{\nu=1}^N \bm{A}\matnot{\nu})^T \cdot (\bigodot_{\nu=1}^N \bm{B}\matnot{\nu}) = (\bm{A}\matnot{1}^T \cdot \bm{B}\matnot{1}) * \ldots * (\bm{A}\matnot{N}^T \cdot \bm{B}\matnot{N})
\label{eq:polygan_suppl_lemma1_N}
\end{equation}
\label{lemma:polygan_lemma_hadamard_kr2}
\end{lemma}

An indicative proof can be found in the Appendix of \citet{chrysos2019polygan}.

\subsection{Scalar output}
\label{ssec:mvp_method_scalar_output_suppl}
The proposed formulation expresses higher-order interactions of the input variables. To elaborate that, we develop the single output case below. That is, we focus on an element $\tau$ of the output vector, e.g., a single pixel. In the next few paragraphs, we consider the case of a scalar output $\outvar_{\tau}$, with $\tau\in [1, o]$ when the input variables are $\sbinvar{I}, \sbinvar{II} \in \knum^d$. To avoid cluttering the notation we only refer to the scalar output with $\outvar_{\tau}$ in the next few paragraphs.

As a reminder, the polynomial of expansion order $N \in \naturalnum$ with output $\boutvar \in \realnum^o$ has the form:

\begin{equation}
    \boutvar = G(\sbinvar{I}, \sbinvar{II}) = \sum_{n=1}^N \sum_{\rho=1}^{n+1} \bigg(\bmcal{W}^{[n, \rho]} \prod_{j=2}^{\rho} \times_{j} \sbinvar{I}  \prod_{\tau=\rho+1}^{n+1} \times_{\tau} \sbinvar{II}\bigg) + \bm{\beta}
    \label{eq:mvp_poly_general_eq_2var_suppl}
\end{equation}

We assume a second order expansion ($N=2$) and let $\tau$ denote an arbitrary scalar output of $\boutvar$. The first order correlations can be expressed through the sums $\sum_{\lambda=1}^d {w_{\tau,\lambda}^{[1,1]}}\invar_{\scalebox{.6}{II},\lambda}$ and $\sum_{\lambda=1}^d {w_{\tau,\lambda}^{[1,2]}}\invar_{\scalebox{.6}{I},\lambda}$. The second order correlations include both auto- and cross-correlations. The tensors $\bmcal{W}^{[2, 1]}$ and $\bmcal{W}^{[2, 3]}$ capture the auto-correlations, while the tensor $\bmcal{W}^{[2, 2]}$ captures the cross-correlations.

\begin{figure}[!t]
    \centering
    \includegraphics[width=0.96\linewidth]{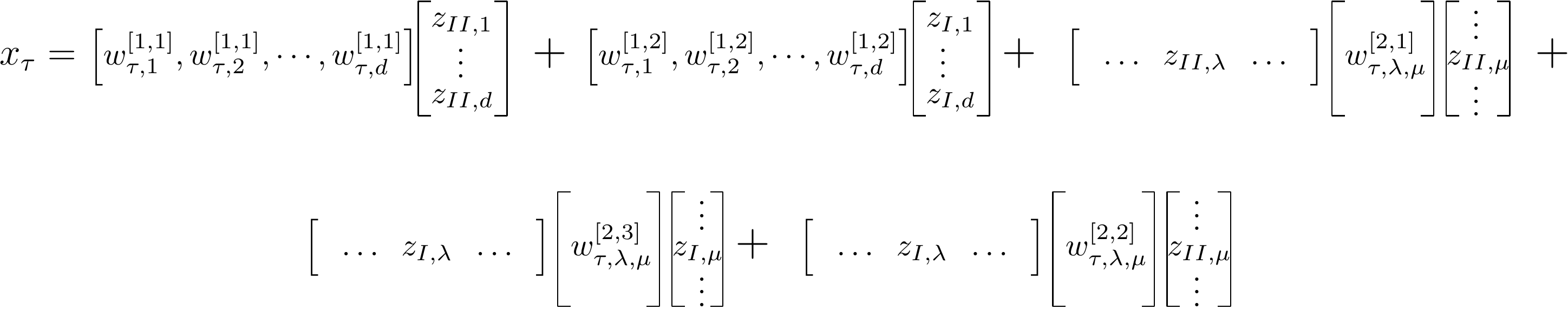}
\caption{Schematic for second order expansion with scalar output $\outvar_{\tau} \in \realnum$. The abbreviations $\invar_{I,\lambda}, \invar_{I,\mu}$ are elements of $\binvar_I$ with $\lambda, \mu \in [1,d]$. Similarly, $\invar_{II,\lambda}, \invar_{II,\mu}$ are elements of  $\binvar_{II}$. The first two terms (on the right side of the equation) are the first-order correlations; the next two terms are the second order auto-correlations. The last term expresses the second order cross-correlations.}
\label{fig:mvp_scalar_output_second_order}
\end{figure} 

A pictorial representation of the correlations are captured in Fig.~\ref{fig:mvp_scalar_output_second_order}. Collecting all the terms in an equation, each output is expressed as:

\begin{equation}
\begin{split}
    \outvar_{\tau} = \beta_\tau + \sum_{\lambda=1}^d \Big[{w_{\tau,\lambda}^{[1,1]}}\invar_{\scalebox{.6}{II},\lambda} + {w_{\tau,\lambda}^{[1,2]}}\invar_{\scalebox{.6}{I},\lambda} + \sum_{\mu=1}^d {w_{\tau,\lambda, \mu}^{[2,1]}}\invar_{\scalebox{.6}{II},\lambda} \invar_{\scalebox{.6}{II},\mu}  + \sum_{\mu=1}^d {w_{\tau,\lambda, \mu}^{[2,3]}}\invar_{\scalebox{.6}{I},\lambda} \invar_{\scalebox{.6}{I},\mu} + \sum_{\mu=1}^d {w_{\tau,\lambda, \mu}^{[2,2]}}\invar_{\scalebox{.6}{I},\lambda} \invar_{\scalebox{.6}{II},\mu} \Big]
\end{split}
\label{eq:mvp_scalar_output_second_order}
\end{equation}

where $\beta_\tau \in \realnum$. Notice that all the correlations of up to second order are captured in \eqref{eq:mvp_scalar_output_second_order}.

\subsection{Second order derivation for two-variable input}
\label{ssec:mvp_method_second_order_proof_ccp}

In all our derivations, the variables associated with the first input $\sbinvar{I}$ have an $I$ notation, e.g., $\bm{U}\matnot{1, I}$. Respectively for the second input $\sbinvar{II}$, the notation $II$ is used.

Even though \eqref{eq:mvp_poly_general_eq_2var_suppl} enables any order of expansion, the learnable parameters increase exponentially, therefore we can use a coupled factorization to reduce the parameters. Next, we derive the factorization for a second order expansion (i.e., $N=2$) and then provide the recursive relationship that generalizes it for an arbitrary order. 

\textbf{Second order derivation}: For a second order expansion (i.e., $N=2$ in \eqref{eq:mvp_poly_general_eq_2var}), we factorize each parameter tensor $\bmcal{W}^{[n, \rho]}$. We assume a coupled CP decomposition for each parameter as follows:

\begin{itemize}
    \item Let $\bm{W}^{[1, 1]}_{(1)} = \bm{C}\bm{U}\matnot{1, II}^T$ and $\bm{W}^{[1, 2]}_{(1)} = \bm{C}\bm{U}\matnot{1, I}^T$ be the parameters for $n=1$.
    \item Let $\bm{W}^{[2, 1]}_{(1)} = \bm{C}(\bm{U}\matnot{2, II} \odot \bm{U}\matnot{1, II})^T$ and  $\bm{W}^{[2, 3]}_{(1)} = \bm{C}(\bm{U}\matnot{2, I} \odot \bm{U}\matnot{1, I})^T$ capture the second order correlations of a single variable ($\sbinvar{II}$ and $\sbinvar{I}$ respectively).
    \item The cross-terms are expressed in $\bmcal{W}^{[2, 2]}\times_{2} \sbinvar{I}\times_{3} \sbinvar{II}$. The output of the $\tau$ element\footnote{An elementwise analysis (with a scalar output) is provided on the \supplementary{} (sec.~\ref{ssec:mvp_method_scalar_output_suppl}).} is $\sum_{\lambda, \mu=1}^d  {w_{\tau,\lambda,\mu}^{[2, 2]}}\invar_{\scalebox{.6}{I},\lambda}\invar_{\scalebox{.6}{II},\mu}$. The product $\bmcal{\hat{W}}^{[2, 2]}\times_{2} \sbinvar{II}\times_{3} \sbinvar{I}$ also results in the same elementwise expression. Hence, to allow for symmetric expression, we factorize the term $\bm{W}^{[2, 2]}_{(1)}$ as the sum of the two terms $\bm{C}(\bm{U}\matnot{2, II} \odot \bm{U}\matnot{1, I})^T$ and $\bm{C}(\bm{U}\matnot{2, I} \odot \bm{U}\matnot{1, II})^T$. For each of the two terms, we assume that the vector-valued inputs are accordingly multiplied. 
\end{itemize}

The parameters $\bm{C} \in \realnum^{o\times k}, \bm{U}\matnot{m, \phi} \in  \realnum^{d\times k}$ ($m=1,2$ and $\phi=\{I, II\}$) are learnable. The aforementioned factorization results in the following equation:

\begin{equation}
  \begin{split}
      \boutvar = \bm{C}\bm{U}\matnot{1, II}^T \sbinvar{II} + \bm{C}\bm{U}\matnot{1, I}^T \sbinvar{I} + \bm{C}\Big(\bm{U}\matnot{2, II} \odot \bm{U}\matnot{1, II}\Big)^T \Big(\sbinvar{II}\odot\sbinvar{II}\Big) + \bm{C}\Big(\bm{U}\matnot{2, I} \odot \bm{U}\matnot{1, I}\Big)^T \Big(\sbinvar{I}\odot\sbinvar{I}\Big) + \\
      \bm{C}\Big(\bm{U}\matnot{2, I} \odot \bm{U}\matnot{1, II}\Big)^T \Big(\sbinvar{I}\odot\sbinvar{II}\Big) + \bm{C}\Big(\bm{U}\matnot{2, II} \odot \bm{U}\matnot{1, I}\Big)^T \Big(\sbinvar{II}\odot\sbinvar{I}\Big) + \bm{\beta}
  \end{split}
\label{eq:mvp_second_order_suppl}
\end{equation}

This expansion captures the correlations (up to second order) of the two input variables $\sbinvar{I}, \sbinvar{II}$.

To make the proof more complete, we remind the reader that the recursive relationship (i.e., (\ref{eq:mvp_model1_rec}) in the main paper) is: 

\begin{equation}
    \boutvar_{n} = \boutvar_{n-1} + \Big(\bm{U}\matnot{n, I}^T \sbinvar{I} +  \bm{U}\matnot{n, II}^T \sbinvar{II}\Big) * \boutvar_{n-1} 
    \label{eq:mvp_model1_rec_suppl}
\end{equation}
for $n=2,\ldots,N$ with $\boutvar_{1} = \bm{U}\matnot{1, I}^T \sbinvar{I} +  \bm{U}\matnot{1, II}^T \sbinvar{II}$ and $\boutvar = \bm{C}\boutvar_{N} + \bm{\beta}$.

\begin{claim}
The equation (\ref{eq:mvp_second_order_suppl}) is a special format of a polynomial that is visualized as in Fig.~\ref{fig:mvp_nth_order_schematic_2var} of the main paper. Equivalently, prove that (\ref{eq:mvp_second_order_suppl}) follows the recursive relationship of (\ref{eq:mvp_model1_rec_suppl}).

\label{claim:mvp_second_order_model1}
\end{claim}

\begin{proof}
We observe that the first two terms of \eqref{eq:mvp_second_order_suppl} are equal to $\bm{C}\boutvar_1$ (from \eqref{eq:mvp_model1_rec_suppl}). By applying Lemma~\ref{lemma:polygan_lemma_hadamard_kr2} in the terms that have Khatri-Rao product, we obtain: 
\begin{equation}
\begin{split}
    \boutvar = \bm{\beta} + \bm{C}\boutvar_1 + \bm{C}\bigg\{ \Big(\bm{U}\matnot{2, II}^T\sbinvar{II}\Big) * \Big(\bm{U}\matnot{1, II}^T\sbinvar{II}\Big) +  \Big(\bm{U}\matnot{2, I}^T\sbinvar{I}\Big) * \Big(\bm{U}\matnot{1, I}^T\sbinvar{I}\Big) + \\
        \Big(\bm{U}\matnot{2, I}^T\sbinvar{I}\Big) * \Big(\bm{U}\matnot{1, II}^T\sbinvar{II}\Big) + \Big(\bm{U}\matnot{2, II}^T\sbinvar{II}\Big) * \Big(\bm{U}\matnot{1, I}^T\sbinvar{I}\Big)\bigg\} = \\
	\bm{\beta} + \bm{C}\boutvar_1 + \bm{C}\bigg\{ \Big[\Big(\bm{U}\matnot{2, I}^T\sbinvar{I}\Big) + \Big(\bm{U}\matnot{2, II}^T\sbinvar{II}\Big)\Big] * \boutvar_1\bigg\} = \bm{C}\boutvar_{2} + \bm{\beta}
\end{split}
\end{equation}

The last equation is precisely the one that arises from the recursive relationship from \eqref{eq:mvp_model1_rec_suppl}.

\end{proof}

To prove the recursive formula for the $N^{th}$ order expansion, a similar pattern as in sec.C of PolyGAN~\citep{chrysos2019polygan} can be followed. Specifically, the difference here is that because of the two input variables, the auto- and cross-correlation variables should be included. Other than that, the same factor sharing is followed.

\subsection{\modeltwo{} model for two-variable input}
\label{ssec:mvp_method_second_order_proof_ncp}

The model proposed above (i.e., \eqref{eq:mvp_model1_rec_suppl}), relies on a single coupled CP decomposition, however a more flexible model can factorize each level with a CP decomposition. To effectively do that, we utilize learnable hyper-parameters $\bm{b}\matnot{n} \in \realnum^\omega$ for $n\in [1, N]$, which act as scaling factors for each parameter tensor. Then, a polynomial of expansion order $N \in \naturalnum$ with output $\boutvar \in \realnum^o$ has the form:

\begin{equation}
    \boutvar = G(\sbinvar{I}, \sbinvar{II}) = \sum_{n=1}^N \sum_{\rho=2}^{n+2} \bigg(\bmcal{W}^{[n, \rho-1]} \times_2 \bm{b}\matnot{N+1-n} \prod_{j=3}^{\rho} \times_{j} \sbinvar{I}  \prod_{\tau=\rho+1}^{n+2} \times_{\tau} \sbinvar{II} \bigg) + \bm{\beta}
    \label{eq:mvp_general_polynomial_with_b}
\end{equation}

To demonstrate the factorization without cluttering the notation, we assume a second order expansion in \eqref{eq:mvp_general_polynomial_with_b}. 

\textbf{Second order derivation}: The second order expansion, i.e., $N=2$, is derived below. We jointy factorize all parameters of \eqref{eq:mvp_general_polynomial_with_b} with a nested decomposition as follows:

\begin{itemize}
    \item First order parameters : $\bm{W}^{[1, 1]}_{(1)} = \bm{C} (\bm{A}\matnot{2, II} \odot \bm{B}\matnot{2})^T$ and $\bm{W}^{[1, 2]}_{(1)} = \bm{C} (\bm{A}\matnot{2, I} \odot \bm{B}\matnot{2})^T$.
    \item Let $\bm{W}^{[2, 1]}_{(1)} = \bm{C} \bigg\{\bm{A}\matnot{2, II} \odot \bigg[\Big(\bm{A}\matnot{1, II} \odot \bm{B}\matnot{1}\Big) \bm{V}\matnot{2}\bigg]\bigg\}^T$ and $\bm{W}^{[2, 3]}_{(1)} = \bm{C} \bigg\{\bm{A}\matnot{2, I} \odot \bigg[\Big(\bm{A}\matnot{1, I} \odot \bm{B}\matnot{1}\Big) \bm{V}\matnot{2}\bigg]\bigg\}^T$ capture the second order correlations of a single variable ($\sbinvar{II}$ and $\sbinvar{I}$ respectively). 
    \item The cross-terms are included in $\bmcal{W}^{[2, 2]}\times_{2} \bm{b}\matnot{1} \times_{3} \sbinvar{I}\times_{4} \sbinvar{II}$. The output of the $\tau$ element is expressed as $\sum_{\nu=1}^{\omega} \sum_{\lambda, \mu=1}^d  {w_{\tau,\nu, \lambda,\mu}^{[2, 2]}}b_{[1], \omega}\invar_{\scalebox{.6}{I},\lambda}\invar_{\scalebox{.6}{II},\mu}$. Similarly, the product $\bmcal{\hat{W}}^{[2, 2]}\times_{2} \bm{b}\matnot{1} \times_{3} \sbinvar{II}\times_{4} \sbinvar{I}$ has output $\sum_{\nu=1}^{\omega} \sum_{\lambda, \mu=1}^d  {w_{\tau,\nu,\mu, \lambda}^{[2, 2]}}b_{[1], \omega}\invar_{\scalebox{.6}{I},\lambda}\invar_{\scalebox{.6}{II},\mu}$ for the $\tau$ element. Notice that the only change in the two expressions is the permutation of the third and forth modes of the tensor; the rest of the expression remains the same. Therefore, to account for this symmetry we factorize the term $\bmcal{W}^{[2, 2]}$ as the sum of two terms and assume that each term is multiplied by the respective terms. Let $\bm{W}^{[2, 2]}_{(1)} = \bm{C} \bigg\{\bm{A}\matnot{2, I} \odot \bigg[\Big(\bm{A}\matnot{1, II} \odot \bm{B}\matnot{1}\Big) \bm{V}\matnot{2}\bigg] + \bm{A}\matnot{2, II} \odot \bigg[\Big(\bm{A}\matnot{1, I} \odot \bm{B}\matnot{1}\Big) \bm{V}\matnot{2}\bigg] \bigg\}^T$.  
\end{itemize}

The parameters $\bm{C} \in  \realnum^{o\times k}, \bm{A}\matnot{n, \phi} \in  \realnum^{d\times k}, \bm{V}\matnot{n} \in  \realnum^{k\times k}, \bm{B}\matnot{n} \in  \realnum^{\omega\times k}$ for $n=1, 2$ and $\phi=\{I, II\}$ are learnable.  
Collecting all the terms above and extracting $\bm{C}$ as a common factor (we ommit $\bm{C}$ below to avoid cluttering the notation):

\begin{equation}
\centering
\begin{split}
    (\bm{A}\matnot{2, II} \odot \bm{B}\matnot{2})^T (\sbinvar{II} \odot \bm{b}\matnot{2}) + (\bm{A}\matnot{2, I} \odot \bm{B}\matnot{2})^T (\sbinvar{I} \odot \bm{b}\matnot{2}) + \\
    \bigg\{\bm{A}\matnot{2, II} \odot \bigg[\Big(\bm{A}\matnot{1, II} \odot \bm{B}\matnot{1}\Big) \bm{V}\matnot{2}\bigg]\bigg\}^T (\sbinvar{II} \odot \sbinvar{II} \odot \bm{b}\matnot{1}) + \\
    \bigg\{\bm{A}\matnot{2, I} \odot \bigg[\Big(\bm{A}\matnot{1, I} \odot \bm{B}\matnot{1}\Big) \bm{V}\matnot{2}\bigg]\bigg\}^T (\sbinvar{I} \odot \sbinvar{I} \odot \bm{b}\matnot{1}) + \\
    \bigg\{\bm{A}\matnot{2, I} \odot \bigg[\Big(\bm{A}\matnot{1, II} \odot \bm{B}\matnot{1}\Big) \bm{V}\matnot{2}\bigg]\bigg\}^T (\sbinvar{I} \odot \sbinvar{II} \odot \bm{b}\matnot{1}) + \\
    \bigg\{\bm{A}\matnot{2, II} \odot \bigg[\Big(\bm{A}\matnot{1, I} \odot \bm{B}\matnot{1}\Big) \bm{V}\matnot{2}\bigg] \bigg\}^T (\sbinvar{II} \odot \sbinvar{I} \odot \bm{b}\matnot{1}) = \\
    \Big(\bm{A}\matnot{2, II}^T \sbinvar{II} + \bm{A}\matnot{2, I}^T \sbinvar{I}\Big) * \Big(\bm{B}\matnot{2}^T \bm{b}\matnot{2}\Big) +  \\
    \Big(\bm{A}\matnot{2, II}^T\sbinvar{II} + \bm{A}\matnot{2, I}^T\sbinvar{I}\Big) * \bigg\{ \bm{V}\matnot{2}^T \bigg[ \Big(\bm{A}\matnot{1, II}^T \sbinvar{II} + \bm{A}\matnot{1, I}^T \sbinvar{I}\Big) * \Big(\bm{B}\matnot{1}^T \bm{b}\matnot{1}\Big) \bigg]  \bigg\}\\
\end{split}
\end{equation}

The last equation is precisely a recursive equation that can be expressed with the Fig.~\ref{fig:mvp_nth_order_schematic_2var_ncp} or equivalently the generalized recursive relationship below.

\begin{figure}[!h]
    \centering
    \includegraphics[width=0.99\linewidth]{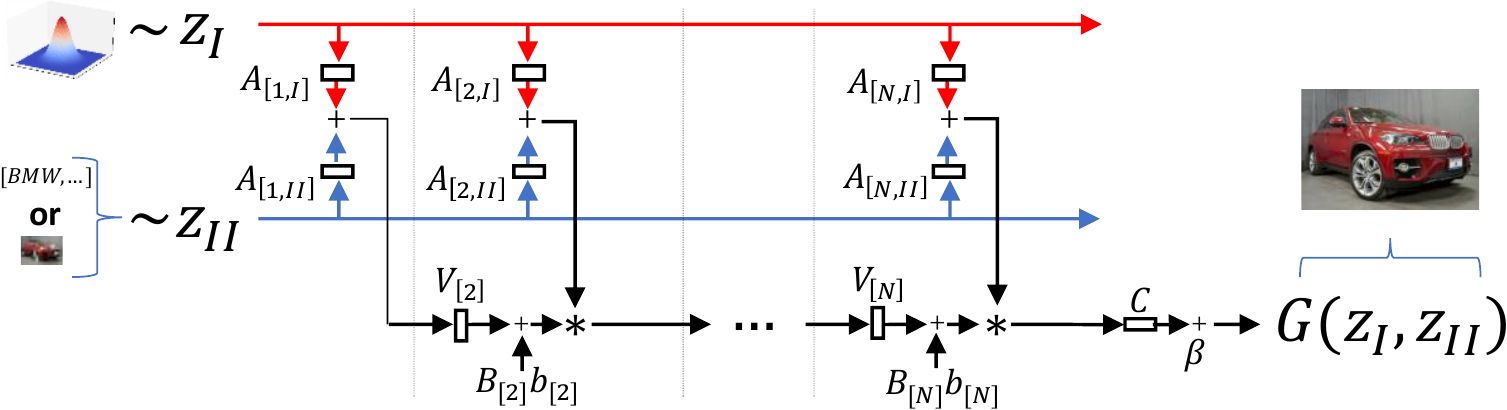}\vspace{2mm} 
\caption{Abstract schematic for $N^{th}$ order approximation of $\boutvar = G(\binvar_{I}, \binvar_{II})$ with \modeltwo{} model. The inputs $\binvar_{I}, \binvar_{II}$ are symmetric in our formulation. We denote with $\binvar_{I}$ a sample from the noise distribution (e.g., Gaussian), while $\binvar_{II}$ symbolizes a sample from a conditional input (e.g., a class label or a low-resolution image). }
\label{fig:mvp_nth_order_schematic_2var_ncp}
\end{figure}

\textbf{Recursive relationship}: The recursive formula for the \modeltwo{} model with arbitrary expansion order $N \in \naturalnum$ is the following:

\begin{equation}
    \boutvar_{n} = \Big(\bm{A}\matnot{n, I}^T \sbinvar{I} +  \bm{A}\matnot{n, II}^T \sbinvar{II}\Big) * \Big(\bm{V}\matnot{n}^T \boutvar_{n-1} + \bm{B}\matnot{n}^T\bm{b}\matnot{n}\Big)
    \label{eq:mvp_model2_rec}
\end{equation}

where $n \in [2, N]$ and $\boutvar_{1} = \Big(\bm{A}\matnot{1, I}^T \sbinvar{I} +  \bm{A}\matnot{1, II}^T \sbinvar{II}\Big) * \Big(\bm{B}\matnot{1}^T\bm{b}\matnot{1}\Big)$. The parameters $\bm{C} \in  \realnum^{o\times k}, \bm{A}\matnot{n, \phi} \in  \realnum^{d\times k}, \bm{V}\matnot{n} \in  \realnum^{k\times k}, \bm{B}\matnot{n} \in  \realnum^{\omega\times k}$ for $\phi=\{I, II\}$ are learnable. Then, the output $\boutvar = \bm{C}\boutvar_{N} + \bm{\beta}$.

The \modeltwo{} model manifests an alternative network that relies on slightly modified assumptions on the decomposition. Thus, changing the underlying assumptions of the decomposition can modify the resulting network. This can be an important tool for domain-specific applications, e.g., when the domain-knowledge should be inserted in the last layers.

\section{Beyond two variables} 
\label{sec:mvp_method_many_inputs_suppl}
Frequently, more than one conditional inputs are required~\citep{yu2018super, xu2017learning, maximov2020ciagan}. In such tasks, the aforementioned framework can be generalized to more than two input variables. We demonstrate how this is possible with three variables; then it can trivially extended to an arbitrary number of input variables.  

Let $\sbinvar{I}, \sbinvar{II}, \sbinvar{III} \in \knum^d$ denote the three input variables. We aim to learn a function that captures the higher-order interactions of the input variables. The polynomial of expansion order $N \in \naturalnum$ with output $\boutvar \in \realnum^o$ has the form: 

\begin{equation}
    \boutvar = G(\sbinvar{I}, \sbinvar{II}, \sbinvar{III}) = \sum_{n=1}^N \sum_{\rho=1}^{n+1} \sum_{\delta=\rho}^{n+1} \bigg(\bmcal{W}^{[n, \rho, \delta]} \prod_{j=2}^{\rho} \times_{j} \sbinvar{I}  \prod_{\tau=\rho+1}^{\delta} \times_{\tau} \sbinvar{II} \prod_{\zeta=\delta+1}^{n+1} \times_{\zeta} \sbinvar{III} \bigg) + \bm{\beta}
    \label{eq:mvp_poly_general_eq_3var}
\end{equation}

where $\bm{\beta} \in \realnum^o$ and $\bmcal{W}^{[n, \rho, \delta]} \in  \realnum^{o\times \prod_{m=1}^{n}\times_m d}$ (for $n\in [1, N]$ and $\rho, \delta\in [1, n+1]$) are the learnable parameters. 
As in the two-variable input, the unknown parameters increase exponentially. To that end, we utilize a joint factorization with factor sharing. The recursive relationship of such a factorization is:

\begin{equation}
    \boutvar_{n} = \boutvar_{n-1} + \Big(\bm{U}\matnot{n, I}^T \sbinvar{I} +  \bm{U}\matnot{n, II}^T \sbinvar{II} + \bm{U}\matnot{n, III}^T \sbinvar{III}\Big) * \boutvar_{n-1} 
    \label{eq:mvp_model1_rec_3var}
\end{equation}
for $n=2,\ldots,N$ with $\boutvar_{1} = \bm{U}\matnot{1, I}^T \sbinvar{I} +  \bm{U}\matnot{1, II}^T \sbinvar{II} + \bm{U}\matnot{1, III}^T \sbinvar{III}$ and $\boutvar = \bm{C}\boutvar_{N} + \bm{\beta}$.

Notice that the pattern (for each order) is similar to the two-variable input: a) a different embedding is found for each input variable, b) the embeddings are added together, c) the result is multiplied elementwise with the representation of the previous order.

\section{Concatenation of inputs}
\label{sec:mvp_concatenation_suppl}

A popular method used for conditional generation is to concatenate the conditional input with the noise labels. However, as we showcase below, concatenation has two significant drawbacks when compared to our framework. To explain those, we will define a concatenation model.

Let $\sbinvar{I} \in \knum_1^{d_1}, \sbinvar{II} \in \knum_2^{d_2}$ where $\knum_1, \knum_2$ can be a subset of real or natural numbers. The output of a concatenation layer is $\outvar = \bm{P}^T \Big[ \sbinvar{I}; \sbinvar{II}\Big]^T$ where the symbol `$;$' denotes the concatenation and $\bm{P} \in \realnum^{(d_1 + d_2)\times o}$ is an affine transformation on the concatenated vector. The $j^{th}$ output is $\outvar_j = \sum_{\tau=1}^{d_1} p_{\tau, j} \invar_{\scalebox{.6}{I},\tau} + \sum_{\tau=1}^{d_2} p_{\tau + d_1, j} \invar_{\scalebox{.6}{II},\tau}$.

Therefore, the two differences from the concatenation case are:
\begin{itemize}
    \item If the input variables are concatenated together we obtain an additive format, not a multiplicative that can capture cross-term correlations. That is, the multiplicative format does allow achieving higher-order auto- and cross- term correlations.
    \item The concatenation changes the dimensionality of the embedding space. Specifically, the input space has dimensionality $d_1 + d_2$. That has a significant toll on the size of the filters (i.e., it increases the learnable parameters), while still having an additive impact. On the contrary, our framework does not change the dimensionality of the embedding spaces. 
\end{itemize}

\section{In-depth differences from \modelnamePI}
\label{sec:mvp_difference_from_pinet_suppl}

In the next few paragraphs, we conduct an in-depth analysis of the differences between \modelnamePI{} and \modelnameMVP. The analysis assumes knowledge of the proposed model, i.e., (\ref{eq:mvp_model1_rec}). 

\citet{chrysos2020poly} introduce \modelnamePI{} as a polynomial expansion of a single input variable. Their goal is to model functions $\boutvar = G(\binvar)$ as high-order polynomial expansions of $\binvar$. Their focus is towards using a single-input variable $\binvar$, which can be noise in case of image generation or an image in discriminative experiments. The authors express the StyleGAN architecture~\citep{karras2018style} as a polynomial expansion, while they advocate that the impressive results can be attributed to the polynomial expansion. 

To facilitate the in-depth analysis, the recursive relationship that corresponds to (\ref{eq:mvp_model1_rec}) is provided below. 
An $N^{th}$ order expansion in \modelnamePI{} is expressed as:
\begin{equation}
    \boutvar_{n} = \Big(\bm{\Lambda}\matnot{n}^T \binvar \Big)* \boutvar_{n-1} + \boutvar_{n-1}
    \label{eq:prodpoly_model1}
\end{equation}

for $n=2,\ldots,N$  with $\boutvar_{1} = \bm{\Lambda}\matnot{1}^T \binvar$ and $\boutvar = \bm{\Gamma}\boutvar_{N} + \bm{\beta}$. The parameters $\bm{\Lambda}, \bm{\Gamma}$ are learnable.

In this work, we focus on conditional data generation, i.e., there are multiple input variables available as auxiliary information. The trivial application of \modelnamePI{} would be to concatenate all the $M$ input variables $\sbinvar{I}, \sbinvar{II}, \sbinvar{III}, \ldots$. The input variable $\binvar$ becomes $\binvar = \Big[\sbinvar{I}; \sbinvar{II}; \sbinvar{III}; \ldots\Big]$, where the symbol `$;$' denotes the concatenation. Then, the polynomial expansion of \modelnamePI{} can be learned on the concatenated $\binvar$. However, there are four significant reasons that we believe that this is not as flexible as the proposed \modelnameMVP.

When we refer to \modelnamePI{} below, we refer to the model with concatenated input. 
In addition, let $\sbinvar{I} \in \knum_1^{d_1}, \sbinvar{II} \in \knum_2^{d_2}$ denote the input variables where $\knum_1, \knum_2$ can be a subset of real or natural numbers. 

\paragraph{Parameter sharing:} \modelnameMVP{} allows additional flexibility in the structure of the architecture, since \modelnameMVP{} utilizes a different projection layer for each input variable. We utilize this flexibility to share the parameters of the conditional input variable; as we detail in (\ref{eq:mvp_sharing_embedding_conditional_variable}), we set $\bm{U}\matnot{n, II} = \bm{U}\matnot{1, II}$ on (\ref{eq:mvp_model1_rec}). 
If we want to perform a similar sharing in \modelnamePI{}, the formulation equivalent to (\ref{eq:prodpoly_model1}) would be $(\lambda\matnot{n})_i = (\lambda\matnot{1})_i$ for $i=d_1,\ldots, d_1 + d_2$. However, sharing only part of the matrix might be challenging. Additionally,  when $\bm{\Lambda}$ is a convolution, the sharing pattern is not straightforward to be computed. 
Therefore, \modelnameMVP{} enables additional flexibility to the model, which is hard to be included in \modelnamePI{}.

\paragraph{Inductive bias:} The inductive bias is crucial in machine learning~\citep{zhao2018bias}, however concatenating the variables restricts the flexibility of the model (i.e. \modelnamePI). 
To illustrate that, let us use the super-resolution experiments as an example.  The input variable $\sbinvar{I}$ is the noise vector and $\sbinvar{II}$ is the (vectorized) low-resolution image. If we concatenate the two variables, then we should use a fully-connected (dense) layer, which does not model well the spatial correlations. Instead, with \modelnameMVP{}, we use a fully-connected layer for the noise vector and a convolution for $\sbinvar{II}$ (low-resolution image). The convolution reduces the number of parameters and captures the spatial correlations in the image. Thus, by concatenating the variables, we reduce the flexibility of the model. 

\paragraph{Dimensionality of the inputs:} The dimensionality of the inputs might vary orders of magnitude, which might create an imbalance during learning. 
For instance, in class-conditional generation concatenating the one-hot labels in the input does not scale well when there are hundreds of classes~\citep{odena2017conditional}. We observe a similar phenomenon in class-conditional generation: in Cars196 (with $196$ classes) the performance of \modelnamePI{} deteriorates considerably when compared to its (relative) performance in CIFAR10 (with $10$ classes). On the contrary, \modelnameMVP{} does not fuse the elements of the input variables directly, but it projects them into a subspace appropriate for adding them. 

\paragraph{Order of expansion with respect to each variable:} Frequently, the two inputs do not require the same order of expansion. Without loss of generality, assume that we need correlations up to $N_{I}$ and $N_{II}$ order (with $N_{I} < N_{II}$) from $\sbinvar{I}$ and $\sbinvar{II}$ respectively. 
\modelnameMVP{} includes a different transformation for each variable, i.e., $\bm{U}\matnot{n, I}$ for $\sbinvar{I}$ and $\bm{U}\matnot{n, II}$ for $\sbinvar{II}$. Then, we can set $\bm{U}\matnot{n, I} = 0$ for $n > N_{I}$. On the contrary, the concatenation of inputs (in \modelnamePI) constrains the expansion to have the same order with respect to each variable.

All in all, we can use concatenation to fuse variables and use \modelnamePI, however an inherently multivariate model is more flexible and can better encode the types of inductive bias required for conditional data generation.

\section{Differences from other networks cast as polynomial neural networks}
\label{sec:mvp_difference_from_other_polynomial_networks_suppl}

A number of networks with impressive results have emerged in (conditional) data generation the last few years. Three such networks that are particularly interesting in our context are \citet{karras2018style, park2019semantic, chen2018self}. We analyze below each method and how it relates to polynomial expansions:

\begin{itemize}
    \item \citet{karras2018style} propose an Adaptive instance normalization (AdaIN) method for unsupervised image generation. An AdaIN layer expresses a second-order interaction\footnote{The formulation is derived from the public implementation of the authors.}: $\bm{h} = (\bm{\Lambda}^T\bm{w}) * n(c(\bm{h}_{in}))$, where $n$ is a normalization, $c$ the convolution operator and $\bm{w}$ is the transformed noise $\bm{w} = MLP(\sbinvar{I})$ (mapping network). The parameters $\bm{\Lambda}$ are learnable, while $\bm{h}_{in}$ is the input to the AdaIN. Stacking AdaIN layers results in a polynomial expansion with a single variable.  
    
    \item \citet{chen2018self} propose a normalization method, called sBN, to stabilize the GAN training. The method performs a `self-modulation’ with respect to the noise variable and optionally the conditional variable in the class-conditional generation setting. Henceforth, we focus on the class-conditional setting that is closer to our work. sBN injects the network layers with a multiplicative interaction of the input variables. Specifically, sBN projects the conditional variable into the space of the variable $\sbinvar{I}$ through an embedding function. Then, the interaction of the two vector-like variables is passed through a fully-connected layer (and a ReLU activation function); the result is injected into the network through the batch normalization parameters.  
    If we cast sBN as a polynomial expansion, it expresses a single polynomial expansion with respect to the input noise and the input conditional variable\footnote{In \modelnameMVP{}, we do not learn a single embedding function for the conditional variable. In addition, we do not project the (transformed) conditional variable to the space of the noise-variable. Both of these can be achieved by making simplifying assumptions on the factor matrices of \modelnameMVP.}. 
    
    \item \citet{park2019semantic} introduce a spatially-adaptive normalization, i.e., \modelnamespade, to improve semantic image synthesis. Their model, referred to as \modelnamespade{} in the remainder of this work, assumes a semantic layout as a conditional input that facilitates the image generation. We analyze in sec.~\ref{sec:mvp_difference_from_spade_suppl} how to obtain the formulation of their spatially-adaptive normalization. If we cast \modelnamespade{} as a polynomial expansion, it expresses a polynomial expansion with respect to the conditional variable. 
\end{itemize}

The aforementioned works propose or modify the batch normalization layer to improve the performance or stabilize the training, while in our work we propose the multivariate polynomial as a general function approximation technique for conditional data generation.  
Nevertheless, given the interpretation of the previous works in the perspective of polynomials, we still can express them as special cases of MVP. Methodologically, there are \textbf{two significant limitations} that none of the aforementioned works tackle:
\begin{itemize}
    \item The aforementioned architectures focus on no or one conditional variable. On the contrary, \modelnameMVP{} naturally extends to \textbf{arbitrarily many conditional variables}.
    
    \item Even though the aforementioned three architectures use (implicitly) a polynomial expansion, a significant factor is the order of the expansion. In our work, the \textbf{product of polynomials} enables capturing higher-order correlations without increasing the amount of layers substantially (sec.~\ref{ssec:mvp_method_extensions}). 
\end{itemize}

In addition to the aforementioned methodological differences, \emph{our work is the only polynomial expansion that conducts experiments on a variety of conditional data generation tasks}. Thus, we both demonstrate methodologically and verify experimentally that \modelnameMVP{} can be used for a wide range of conditional data generation tasks.

\subsection{In-depth differences from \modelnamespade}
\label{sec:mvp_difference_from_spade_suppl}

In the next few paragraphs, we conduct an in-depth analysis of the differences between \modelnamespade{} and \modelnameMVP. 

\citet{park2019semantic} introduce a spatially-adaptive normalization, i.e., \modelnamespade, to improve semantic image synthesis. Their model, referred to as \modelnamespade{} in the remainder of this work, assumes a semantic layout as a conditional input that facilitates the image generation. 

The $n^{th}$ model block applies a normalization on the representation $\boutvar_{n-1}$ of the previous layer and then it performs an elementwise multiplication with a transformed semantic layout. The transformed semantic layout can be denoted as $\bm{A}\matnot{n, II}^T \sbinvar{II}$ where $\sbinvar{II}$ denotes the conditional input to the generator. The output of this elementwise multiplication is then propagated to the next model block that performs the same operations. Stacking $N$ such blocks results in an $N^{th}$ order polynomial expansion which is expressed as:

\begin{equation}
    \boutvar_{n} = \Big(\bm{A}\matnot{n, II}^T \sbinvar{II}\Big) * \Big(\bm{V}\matnot{n}^T \boutvar_{n-1} + \bm{B}\matnot{n}^T\bm{b}\matnot{n}\Big)
    \label{eq:mvp_model2_rec_spade}
\end{equation}

where $n \in [2, N]$ and $\boutvar_{1} = \bm{A}\matnot{1, I}^T \sbinvar{I}$. The parameters $\bm{C} \in  \realnum^{o\times k}, \bm{A}\matnot{n, \phi} \in  \realnum^{d\times k}, \bm{V}\matnot{n} \in  \realnum^{k\times k}, \bm{B}\matnot{n} \in  \realnum^{\omega\times k}$ for $\phi=\{I, II\}$ are learnable. Then, the output $\boutvar = \bm{C}\boutvar_{N} + \bm{\beta}$.

\modelnamespade{} as expressed in (\ref{eq:mvp_model2_rec_spade}) resembles one of the proposed models of \modelnameMVP{} (specifically (\ref{eq:mvp_model2_rec})). In particular, it expresses a polynomial with respect to the conditional variable. The parameters $\bm{A}\matnot{n, I}$ are set as zero, which means that there are no higher-order correlations with respect to the input variable $\sbinvar{I}$. Therefore, our work bears the following differences from \citet{park2019semantic}:
\begin{itemize}
    \item \modelnamespade{} proposes a normalization scheme that is only applied to semantic image generation. On the contrary, our proposed \modelnameMVP{} can be applied to any conditional data generation task, e.g., class-conditional generation or image-to-image translation. 
    
    \item \textbf{\modelnamespade{} is a special case of \modelnameMVP}. In particular, by setting i) $\bm{A}\matnot{1, II}$ equal to zero, ii)  $\bm{A}\matnot{n, I}$ in (\ref{eq:mvp_model2_rec}) equal to zero, we obtain \modelnamespade. In addition, \modelnameMVP{} allows different assumptions on the decompositions which lead to an alternative structure, such as (\ref{eq:mvp_model1_rec}).
    
    \item \modelnamespade{} proposes a polynomial expansion with respect to a single variable. On the other hand, our model can extend to an arbitrary number of input variables to account for auxiliary labels, e.g., (\ref{eq:mvp_model1_rec_3var}). 
    
    \item Even though \modelnamespade{} models higher-order correlations of the conditional variable, it still does not leverage the higher-order correlations of the representations (e.g., as in the product of polynomials) and hence without activation functions it might not work as well as the two-variable expansion. 
\end{itemize}

\citet{park2019semantic} exhibit impressive generation results with large-scale computing (i.e., they report results using NVIDIA DGX with 8 V100 GPUs). Our goal is not to compete in computationally heavy, large-scale experiments, but rather to illustrate the benefits of the generic formulation of \modelnameMVP. 

\modelnamespade{} is an important baseline for our work. In particular, we augment \modelnamespade{} in wo ways: a) by extending it to accept both continuous and discrete variables in $\sbinvar{II}$ and b) by adding polynomial terms with respect to the input variable $\sbinvar{I}$. The latter model is referred to as \modelnamespade-\modelnameMVP{} (details on the next section).

 \section{Experimental details}
\label{sec:mvp_experimental_details_suppl}

\paragraph{Metrics:} The two most popular metrics~\citep{lucic2018gans, creswell2018generative} for evaluation of the synthesized images are the Inception Score (IS)~\citep{salimans2016improved} and the Frechet Inception Distance (FID)~\citep{heusel2017gans}. The metrics utilize the pretrained Inception network~\citep{szegedy2015going} to extract representations of the synthesized images. FID assumes that the representations extracted follow a Gaussian distribution and matches the statistics (i.e., mean and variance) of the representations between real and synthesized samples. Alternative evaluation metrics have been reported as inaccurate, e.g., in \citet{theis2015note}, thus we use the IS and FID. Following the standard practice of the literature, the IS is computed by synthesizing $5,000$ samples, while the FID is computed using $10,000$ samples. 

The IS is used exclusively for images of natural scenes as a metric. The reasoning behind that is that the Inception network has been trained on images of natural scenes. On the contrary, the FID metric relies on the first and second-order moments of the representations, which are considered more robust to different types of images. Hence, we only report IS for the CIFAR10 related experiments, while for the rest the FID is reported.

\paragraph{Dataset details:} There are eight datasets used in this work:
\begin{itemize}
    \item Large-scale CelebFaces Attributes (or \emph{CelebA} for short)~\citep{liu2015deep} is a large-scale face attributes dataset with $202,000$ celebrity images. We use $160,000$ images for training our method. 
    \item \emph{Cars196}~\citep{krause20133d} is a dataset that includes different models of cars in different positions and backgrounds. Cars196 has $16,000$ images, while the images have substantially more variation than CelebA faces. 
    \item \emph{CIFAR10}~\citep{krizhevsky2014cifar} contains $60,000$ images of natural scenes. Each image is of resolution $32\times32\times3$ and is classified in one of the $10$ classes. CIFAR10 is frequently used as a benchmark for image generation.  
    \item The Street View House Numbers dataset (or \emph{SVHN} for short)~\citep{netzer2011reading} has $100,000$ images of digits ($73,257$ of which for training). SVHN includes color house-number images which are classified in $10$ classes; each class corresponds to a digit $0$ to $9$. SVHN images are diverse (e.g., with respect to background, scale).
    \item \emph{MNIST}~\citep{lecun1998gradient} consists of images with handwritten digits. Each images depicts a single digit (annotated from $0$ to $9$) in a $28\times 28$ resolution. The dataset includes $60,000$ images for training.
    \item \emph{Shoes}~\citep{fine-grained, xie15hed} consists of $50,000$ images of shoes, where the edges of each shoe are extracted~\citep{isola2016image}.
    \item \emph{Handbags}~\citep{zhu2016generative, xie15hed} consists of more than $130,000$ images of handbag items. The edges have been computed for each image and used as conditional input to the generator~\citep{isola2016image}.
    \item \emph{Anime characters} dataset~\citep{jin2017towards} consists of anime characters that are generated based on specific attributes, e.g., hair color. The public version used\footnote{The version is downloaded following the instructions of \url{https://github.com/bchao1/Anime-Generation}.} contains annotations on the hair color and the eye color. We consider $7$ classes on the hair color and $6$ classes on the eye color, with a total of $14,000$ training images. 
\end{itemize}

All the images of CelebA, Cars196, Shoes and Handbags are resized to $64\times64$ resolution.

\paragraph{Architectures:} The discriminator structure is left the same for each experiment, we focus only on the generator architecture. All the architectures are based on two different generator schemes, i.e., the SNGAN~\citep{miyato2018cgans} and the polynomial expansion of \citet{chrysos2020poly} that does not include activation functions in the generator. 

The variants of the generator of SNGAN are described below:
\begin{itemize}
    \item \textbf{SNGAN}~\citep{miyato2018cgans}: The generator consists of a convolution, followed by three residual blocks. The discriminator is also based on successive residual blocks. The public implementation of SNGAN with conditional batch normalization (CBN) is used as the baseline.
    
    \item \textbf{\modelnameMVPsngan} [proposed]: We convert the resnet-based generator of SNGAN to an \modelnameMVP{} model. To obtain \modelnameMVP, the SNGAN is modified in two ways: a) the Conditional Batch Normalization (CBN) is converted into batch normalization~\citep{ioffe2015batch}, b) the injections of the two embeddings (from the inputs) are added after each residual block, i.e. the formula of (\ref{eq:mvp_model1_rec}). In other words, the generator is converted to a product of two-variable polynomials. 
    
    \item \textbf{SNGAN-CONC}: Based on \modelnameMVPsngan, we replace each Hadamard product with a concatenation. This implements the variant mentioned in sec.~\ref{sec:mvp_concatenation_suppl}.
    
    \item \textbf{SNGAN-\modelnamespade}~\citep{park2019semantic}: As described in sec.~\ref{sec:mvp_difference_from_spade_suppl}, \modelnamespade{} is a polynomial with respect to the conditional variable $\sbinvar{II}$. The generator of \modelnameMVPsngan{} is modified to perform the Hadamard product with respect to the conditional variable every time.
    
\end{itemize}

The variants of the generator of \modelnamePI{} are described below:
\begin{itemize}
    \item \textbf{\modelnamePI}~\citep{chrysos2020poly}: The generator is based on a product of polynomials. The first polynomials use fully-connected connections, while the next few polynomials use cross-correlations. The discriminator is based on the residual blocks of SNGAN. We stress out that the generator does not include any activation functions apart from a hyperbolic tangent in the output space for normalization. The authors advocate that this exhibits the expressivity of the designed model. 

    \item \textbf{\modelnamePI-SICONC}: The generator structure is based on \modelnamePI{} with two modifications: a) the Conditional Batch Normalization is converted into batch normalization~\citep{ioffe2015batch}, b) the second-input is concatenated with the first (i.e., the noise) in the input of the generator. Thus, this is a single variable polynomial, i.e., a \modelnamePI{}, where the second-input is vectorized and concatenated with the first. This baseline implements the \modelnamePI{} described in sec.~\ref{sec:mvp_difference_from_pinet_suppl}. 
    
    \item \textbf{\modelnameMVP} [proposed]: The generator of \modelnamePI{} is converted to an \modelnameMVP{} model with two modifications: a) the Conditional Batch Normalization is converted into batch normalization~\citep{ioffe2015batch}, b) instead of having a Hadamard product with a single variable as in \modelnamePI{}, the formula with the two-variable input (e.g., (\ref{eq:mvp_model1_rec})) is followed.
    
    \item \textbf{GAN-CONC}: Based on \modelnameMVP, each Hadamard product is replaced by a concatenation. This implements the variant mentioned in sec.~\ref{sec:mvp_concatenation_suppl}. 
    
    \item \textbf{GAN-ADD}: Based on \modelnameMVP, each Hadamard product is replaced by an addition. This modifies (\ref{eq:mvp_model2_rec}) to $\boutvar_{n} = \Big(\bm{A}\matnot{n, I}^T \sbinvar{I} +  \bm{A}\matnot{n, II}^T \sbinvar{II}\Big) + \Big(\bm{V}\matnot{n}^T \boutvar_{n-1} + \bm{B}\matnot{n}^T\bm{b}\matnot{n}\Big)$. 
    
    \item \textbf{\modelnamespade}~\citep{park2019semantic}: As described in sec.~\ref{sec:mvp_difference_from_spade_suppl}, \modelnamespade{} defines a polynomial with respect to the conditional variable $\sbinvar{II}$. The generator of \modelnamePI{} is modified to perform the Hadamard product with respect to the conditional variable every time. 
    
    \item \textbf{\modelnamespade-\modelnameMVP} [proposed]: This is a variant we develop to bridge the gap between \modelnamespade{} and the proposed \modelnameMVP. Specifically, we augment the aforementioned \modelnamespade{} twofold: a) the dense layers in the input space are converted into a polynomial with respect to the variable $\sbinvar{I}$ and b) we also convert the polynomial in the output (i.e., the rightmost polynomial in the Fig.~\ref{fig:mvp_schematics_product_polynomials} schematics) to a polynomial with respect to the variable $\sbinvar{I}$. This model captures higher-order correlations of the variable  $\sbinvar{I}$ that \modelnamespade{} did not not originally include. This model still includes single variable polynomials, however the input in each polynomial varies and is not only the conditional variable.
    
\end{itemize}

We should note that StyleGAN is a special case of \modelnamePI{} as demonstrated by \citet{chrysos2020poly}, and converting it into a products of polynomials formulation improves its performance. Hence, we use directly the \modelnamePI. 
The two baselines GAN-CONC and GAN-ADD capture only additive correlations, hence they cannot effectively model complex distributions without activation functions. Nevertheless, they are added as a reference point to emphasize the benefits of higher-order polynomial expansions. 

An abstract schematic of the generators that are in the form of products of polynomials is depicted in Fig.~\ref{fig:mvp_schematics_product_polynomials}. Notice that the compared methods from the literature use polynomials of a single variable, while we propose a polynomial with an arbitrary number of inputs (e.g., two-input shown in the schematic).

\begin{figure}[!h]
\centering
    \subfloat[\modelnamePI~\citep{chrysos2020poly}]{\includegraphics[width=0.99\linewidth]{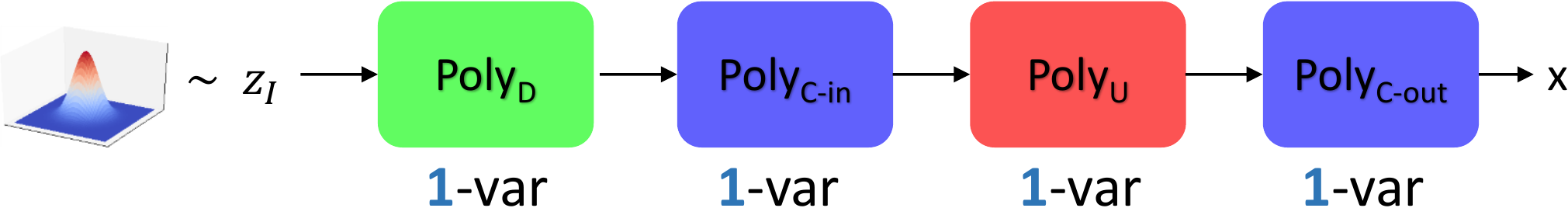}}\\
    \subfloat[\modelnamePI-SICONC (baseline)]{\includegraphics[width=0.99\linewidth]{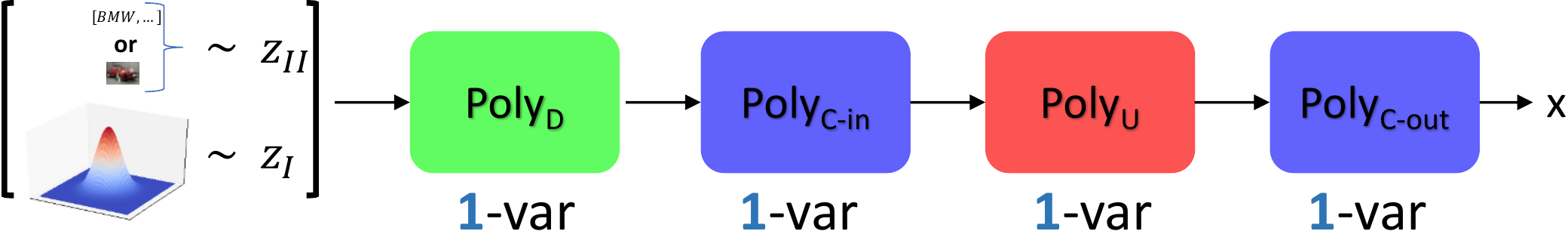}}\\
    \subfloat[\modelnamespade~\citep{park2019semantic}]{\includegraphics[width=0.99\linewidth]{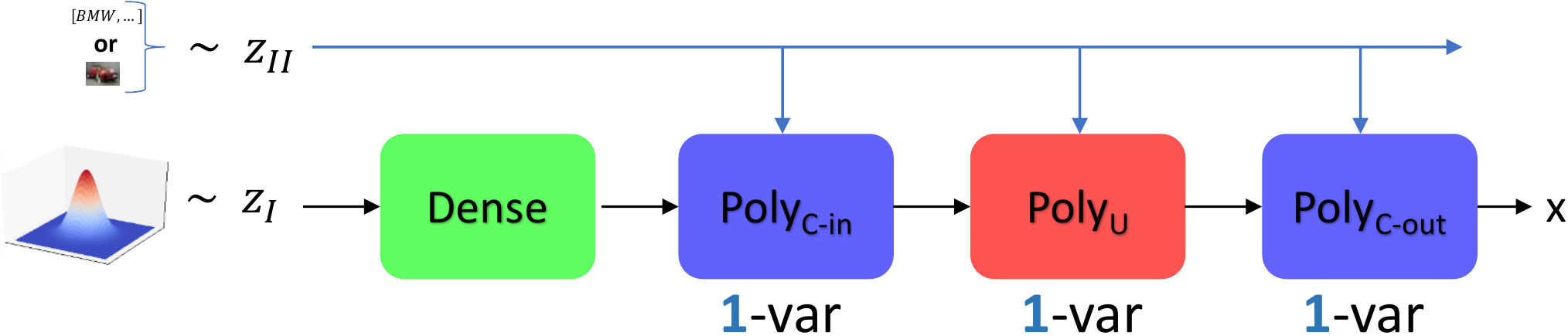}}\\
    \subfloat[\modelnameMVP{} (proposed)]{\includegraphics[width=0.99\linewidth]{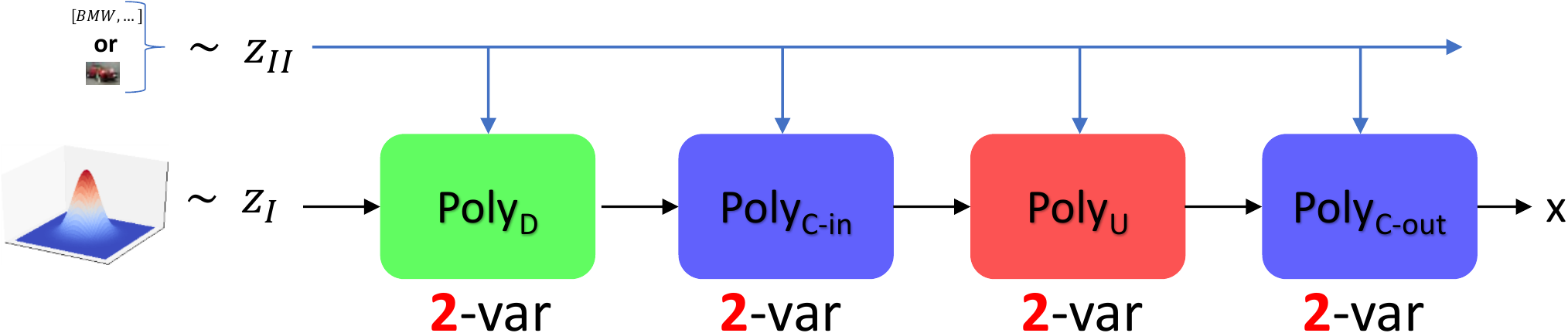}}
\caption{Abstract schematic of the different compared generators. All the generators are products of polynomials. Each colored box represents a different type of polynomial, i.e., the \textcolor{green}{green box} symbolizes polynomial(s) with dense layers, the \textcolor{blue}{blue box} denotes convolutional or cross-correlation layers. The \textcolor{red}{red box} includes the up-sampling layers. (a) \modelnamePI{} implements a single-variable polynomial for modeling functions $\boutvar = G(\binvar)$. \modelnamePI{} enables class-conditional generation by using conditional batch normalization (CBN). (b) An alternative to CBN is to concatenate the conditional variable in the input, as in \modelnamePI-SICONC. This also enables the non-discrete conditional variables (e.g., low-resolution images) to be concatenated. (c) \modelnamespade{} implements a single-variable polynomial expansion with respect to the conditional variable $\binvar_{II}$. This is substantially different from the polynomial with multiple-input variables, i.e., \modelnameMVP{}. Two additional differences are that (i) \modelnamespade{} is motivated as a spatially-adaptive method (i.e., for continuous conditional variables), while \modelnameMVP{} can be used for diverse types of conditional variables, (ii) there is no polynomial in the dense layers in the \modelnamespade. However, as illustrated in \modelnamePI{} converting the dense layers into a higher-order polynomial can further boost the performance. (d) The proposed generator structure.}
\label{fig:mvp_schematics_product_polynomials}
\end{figure} 
\paragraph{Implementation details of \modelnameMVP:} Throughout this work, we reserve the symbol $\sbinvar{II}$ for the conditional input (e.g., a class label). In each polynomial, we reduce further the parameters by using the same embedding for the conditional variables. That is expressed as:
\begin{equation}
    \bm{U}\matnot{n, II} = \bm{U}\matnot{1, II}
    \label{eq:mvp_sharing_embedding_conditional_variable}
\end{equation}

for $n=2,\ldots,N$. Equivalently, that would be $\bm{A}\matnot{n, II} = \bm{A}\matnot{1, II}$ in (\ref{eq:mvp_model2_rec}).
Additionally, \modeltwo{}  performed better in our preliminary experiments, thus we use (\ref{eq:mvp_model2_rec}) to design each polynomial.
Given the aforementioned sharing, the $N^{th}$ order expansion is described by:
\begin{equation}
    \boutvar_{n} = \Big(\bm{A}\matnot{n, I}^T \sbinvar{I} +  \bm{A}\matnot{\textcolor{red}{1}, II}^T \sbinvar{II}\Big) * \Big(\bm{V}\matnot{n}^T \boutvar_{n-1} + \bm{B}\matnot{n}^T\bm{b}\matnot{n}\Big)
    \label{eq:mvp_model2_rec_sharing_embedding}
\end{equation}
for $n=2,\ldots,N$. Lastly, the factor $\bm{A}\matnot{1, II}$ is a convolutional layer when the conditional variable is an image, while it is a fully-connected layer otherwise.

\paragraph{Order of each polynomial in \modelnameMVP:}: Our experimental analysis demonstrates that there is not a single order of the polynomial expansion that works in all cases. The different polynomial expansions used (cf. Fig.~\ref{fig:mvp_schematics_product_polynomials}) utilize different orders of expansion. The first set of polynomials with fully-connected layers typically includes low-order polynomials (i.e., $1^{\text{st}}-3^{\text{rd}}$ order), while the main polynomial includes $6^{\text{th}}-9^{\text{th}}$ degree polynomial, and the output polynomial is also $1^{\text{st}}-3^{\text{rd}}$ order polynomial.

 \section{Additional experiments}
\label{sec:mvp_experiments_suppl}

Additional experiments and visualizations are provided in this section. Additional visualizations for class-conditional generation are provided in sec.~\ref{ssec:mvp_experiments_visualization_suppl}. An additional experiment with class-conditional generation with SVHN digits is performed in sec.~\ref{ssec:mvp_experiment_class_conditional_svhn_suppl}. An experiment that learns the translation of MNIST to SVHN digits is conducted in sec.~\ref{ssec:mvp_experiment_translation_mnist_svhn_suppl}. To explore further the image-to-image translation, two additional experiments are conducted in sec.~\ref{ssec:mvp_experiment_translation_edge2im_suppl}. An attribute-guided generation is performed in sec.~\ref{ssec:mvp_experiment_multiple_discrete_conditional_input_suppl} to illustrate the benefit of our framework with respect to multiple, discrete conditional inputs. This is further extended in sec.~\ref{ssec:mvp_experiment_mixed_conditional_input_suppl}, where an experiment with mixed conditional input is conducted. Finally, an additional diversity-inducing regularization term is used to assess whether it can further boost the diversity the synthesized images in sec.~\ref{ssec:mvp_experiment_translation_edge2im__with_diversity_suppl}.

\begin{figure*}
\centering
    \centering
    \subfloat[Super-resolution $8\times$]{\includegraphics[width=0.32\linewidth]{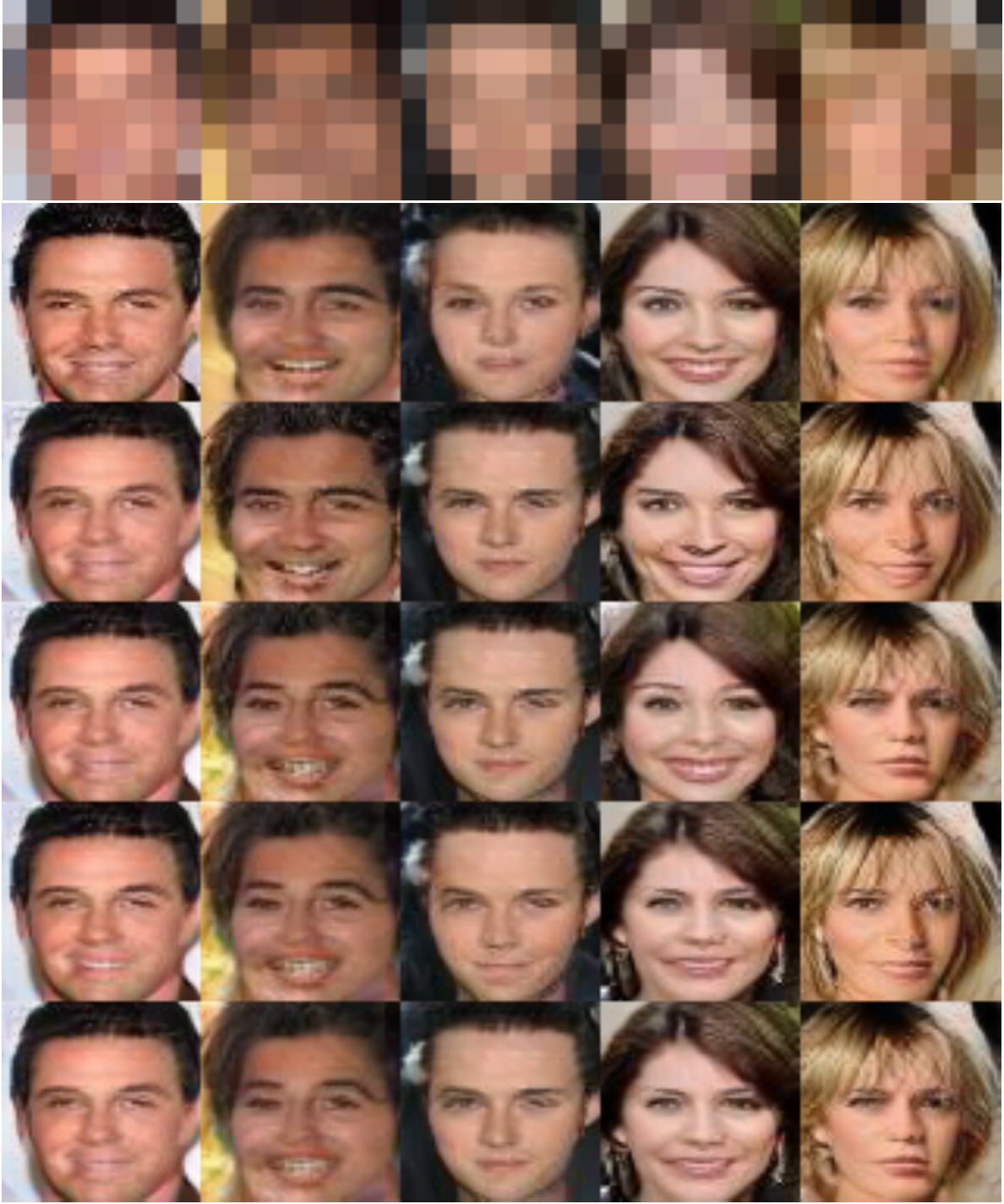}\hspace{1mm}}
    \subfloat[Super-resolution $8\times$]{\includegraphics[width=0.32\linewidth]{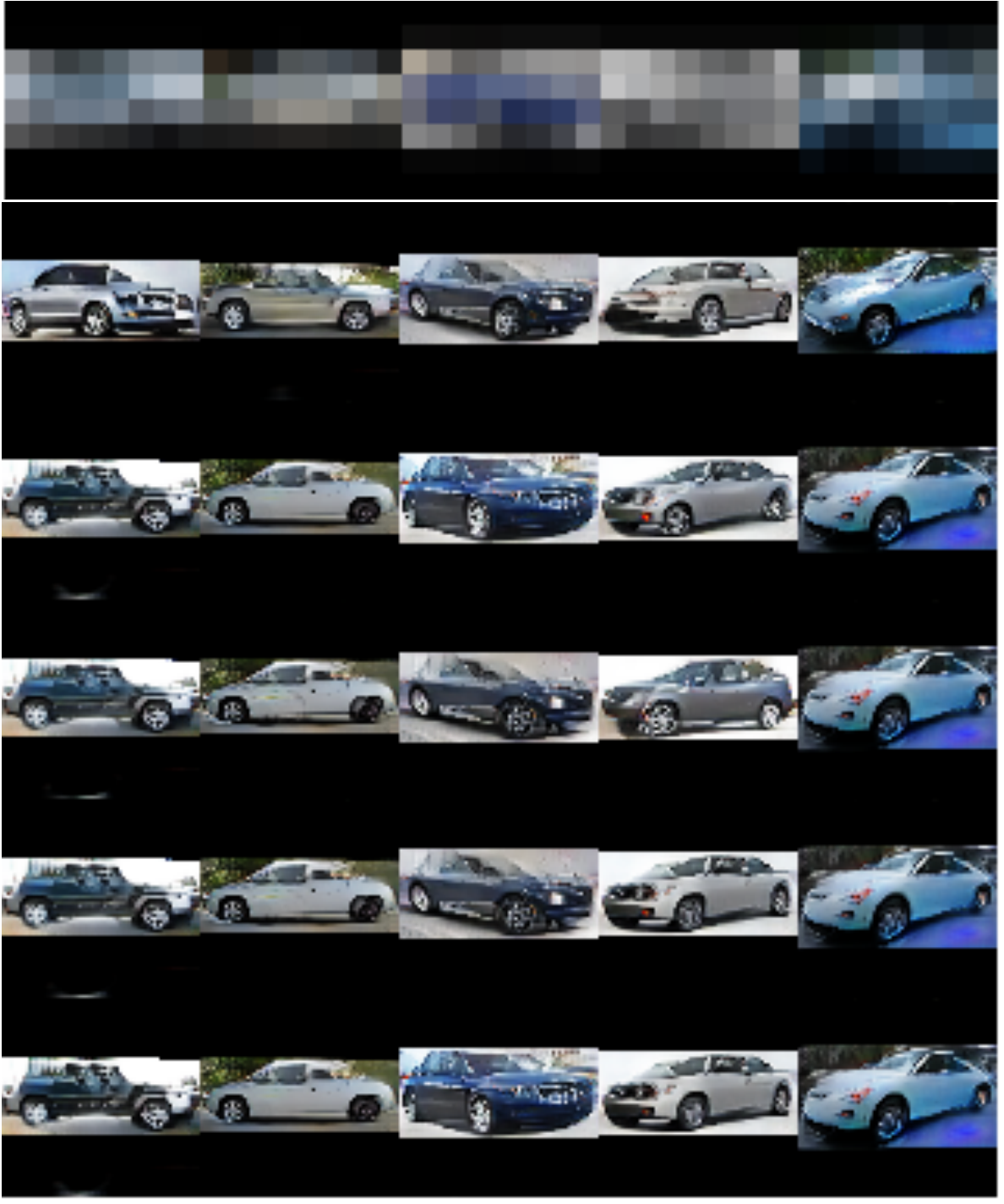}\hspace{1mm}}
    \subfloat[Super-resolution $16\times$]{\includegraphics[width=0.32\linewidth]{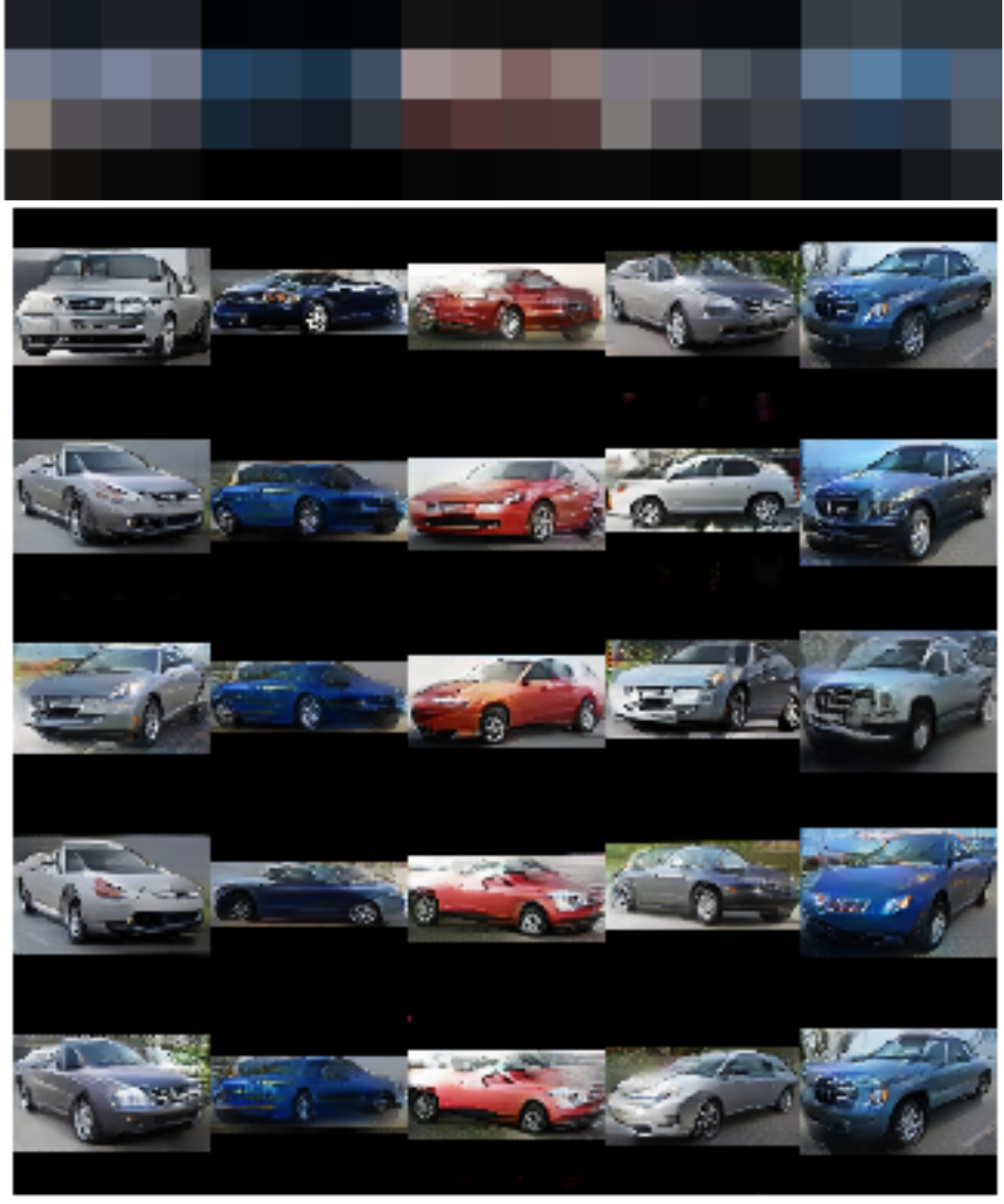}}
\caption{Synthesized images for super-resolution by (a), (b) $8\times$, (c) $16\times$. The first row depicts the conditional input (i.e., low-resolution image). The rows 2-6 depict outputs of the \modelnameMVP{} when a noise vector is sampled per row. Notice how the noise changes (a) the smile or the pose of the head, (b) the color, car type or even the background, (c) the position of the car. }
\label{fig:mvp_qualitative_continuous_cond}
\end{figure*}

\begin{figure*}
\centering
    \centering
    \subfloat[Female+Neutral]{\includegraphics[width=0.495\linewidth]{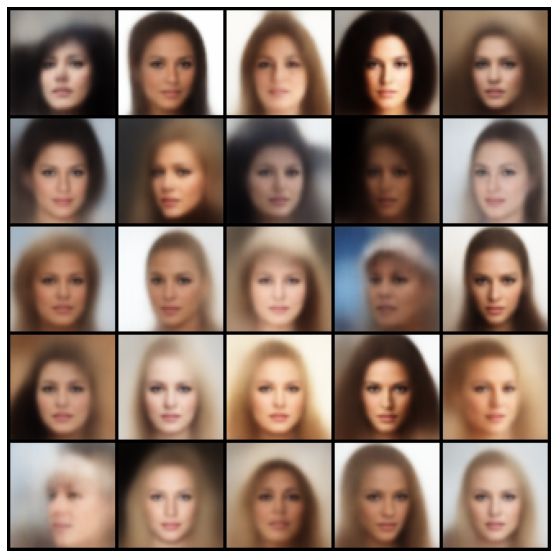}\hspace{1mm}}
    \subfloat[Female+Smile (Unseen)]{\includegraphics[width=0.495\linewidth]{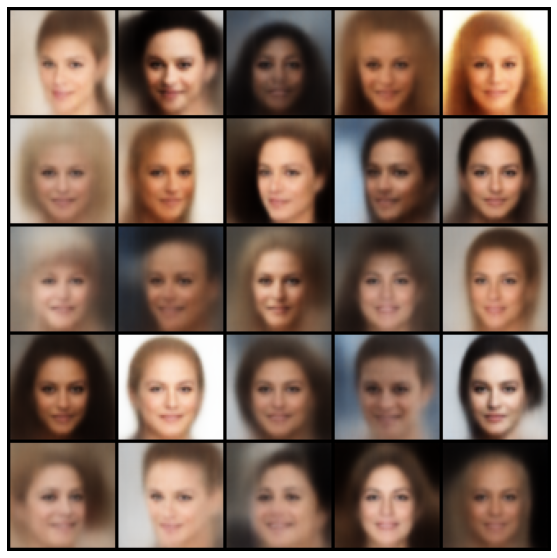}\hspace{1mm}}\\
    \subfloat[Male+Neutral]{\includegraphics[width=0.495\linewidth]{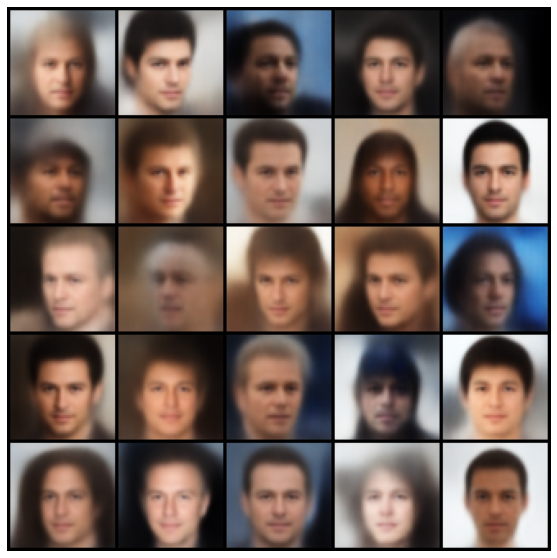}\hspace{1mm}}
    \subfloat[Male+Smile]{\includegraphics[width=0.495\linewidth]{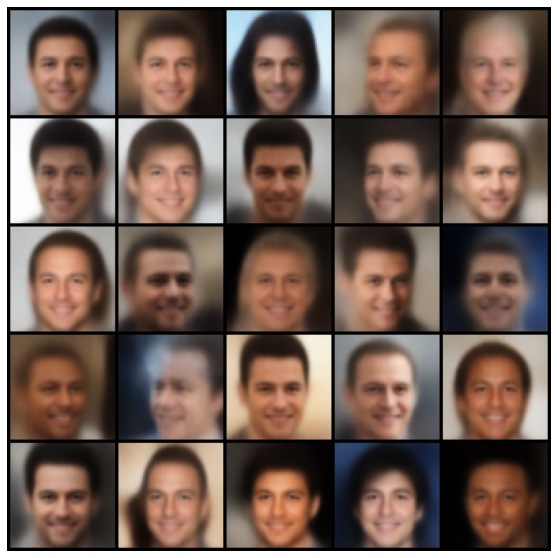}}
\caption{Additional synthesized images for \modelnameMVP-VAE. The setup of sec.~\ref{ssec:mvp_experiments_unseen_combinations} is used to provide the combinations, where the Woman+Smile is the combination not included in the training set. Notice that the proposed \modelnameMVP-VAE can synthesize images from all four combinations.}
\label{fig:mvp_vae_missing_combinations_suppl}
\end{figure*}

\subsection{Additional visualizations in class-conditional generation}
\label{ssec:mvp_experiments_visualization_suppl}

In Fig.~\ref{fig:mvp_qualitative_cifar10_comparison} the qualitative results of the compared methods in class-conditional generation on CIFAR10 are shared. Both the generator of SNGAN and ours have activation functions in this experiment. 

\begin{figure}[!h]
\centering
    \centering
    \subfloat[Ground-truth samples]{\includegraphics[width=0.33\linewidth]{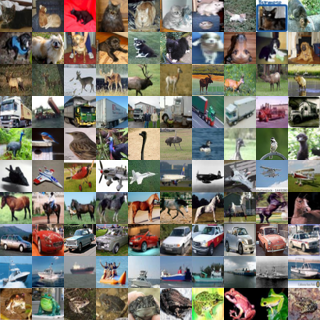}\hspace{1mm}}
    \subfloat[SNGAN~\citep{miyato2018spectral}]{\includegraphics[width=0.322\linewidth]{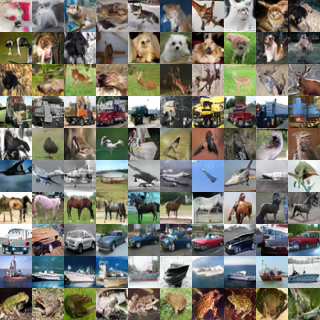}\hspace{1mm}}
    \subfloat[\modelnameMVP]{\includegraphics[width=0.322\linewidth]{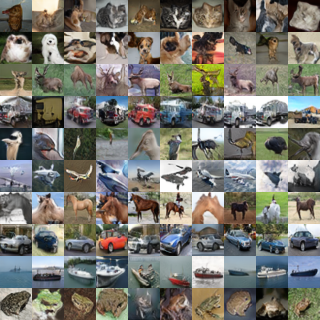}\hspace{1mm}}
\caption{Qualitative results on CIFAR10. Each row depicts random samples from a single class.}
\label{fig:mvp_qualitative_cifar10_comparison}
\end{figure}

In Fig.~\ref{fig:mvp_qualitative_linear_cifar10_comparison} samples from the baseline \modelnamePI~\citep{chrysos2020poly} and our method are depicted for the class-conditional generation on CIFAR10. The images have a substantial difference. Similarly, in Fig.~\ref{fig:mvp_qualitative_linear_cars_comparison} a visual comparison between \modelnamePI{} and \modelnameMVP{} is exhibited in Cars196 dataset. To our knowledge, no framework in the past has demonstrated such expressivity; \modelnameMVP{} synthesizes images that approximate the quality of synthesized images from networks with activation functions.

\begin{figure}[!h]
\centering
    \centering
    \subfloat[\modelnamePI~\citep{chrysos2020poly}]{\includegraphics[width=0.495\linewidth]{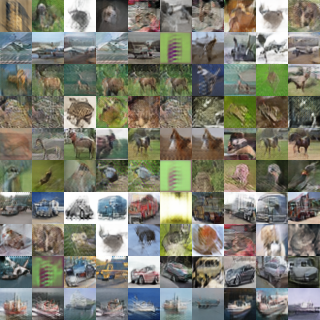}\hspace{1mm}}
    \subfloat[\modelnameMVP]{\includegraphics[width=0.495\linewidth]{qualitative-cifar10_linear_class_cond.png}\hspace{1mm}}
\caption{Qualitative results on CIFAR10 when the generator does not include activation functions between the layers. Each row depicts random samples from a single class; the same class is depicted in each pair of images. For instance, the last row corresponds to boats. }
\label{fig:mvp_qualitative_linear_cifar10_comparison}
\end{figure}

\begin{figure}[!h]
\centering
    \centering
    \subfloat[\modelnamePI]{\includegraphics[width=0.495\linewidth]{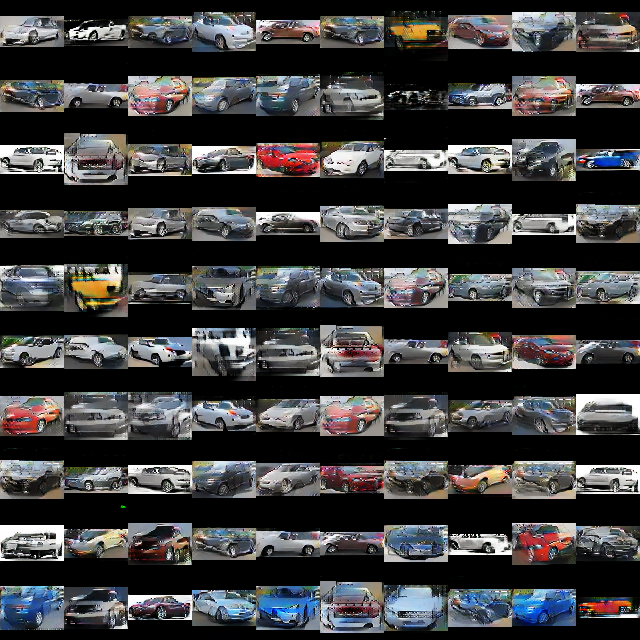}\hspace{1mm}}
    \subfloat[\modelnameMVP]{\includegraphics[width=0.495\linewidth]{qualitative-cars_linear_class_cond.png}\hspace{1mm}}
\caption{Qualitative results on Cars196 when the generator does not include activation functions between the layers. Each row depicts cherry-picked samples from a single class; the same class is depicted in each pair of images. The differences between the synthesized images are dramatic.}
\label{fig:mvp_qualitative_linear_cars_comparison}
\end{figure}

In Fig.~\ref{fig:mvp_comparisons_interclass_cifar10}, an inter-class interpolation of various compared methods in CIFAR10 are visualized. The illustrations of the intermediate images in SNGAN-CONC and SNGAN-ADD are either blurry or not realistic. On the contrary, in \modelnamespade{} and \modelnameMVP{} the higher-order polynomial expansion results in more realistic intermediate images. Nevertheless, \modelnameMVP{} results in sharper shapes and images even in the intermediate results when compared to \modelnamespade.

\begin{figure}[h]
\centering
    \centering
    \subfloat[SNGAN-CONC]{\includegraphics[width=0.495\linewidth]{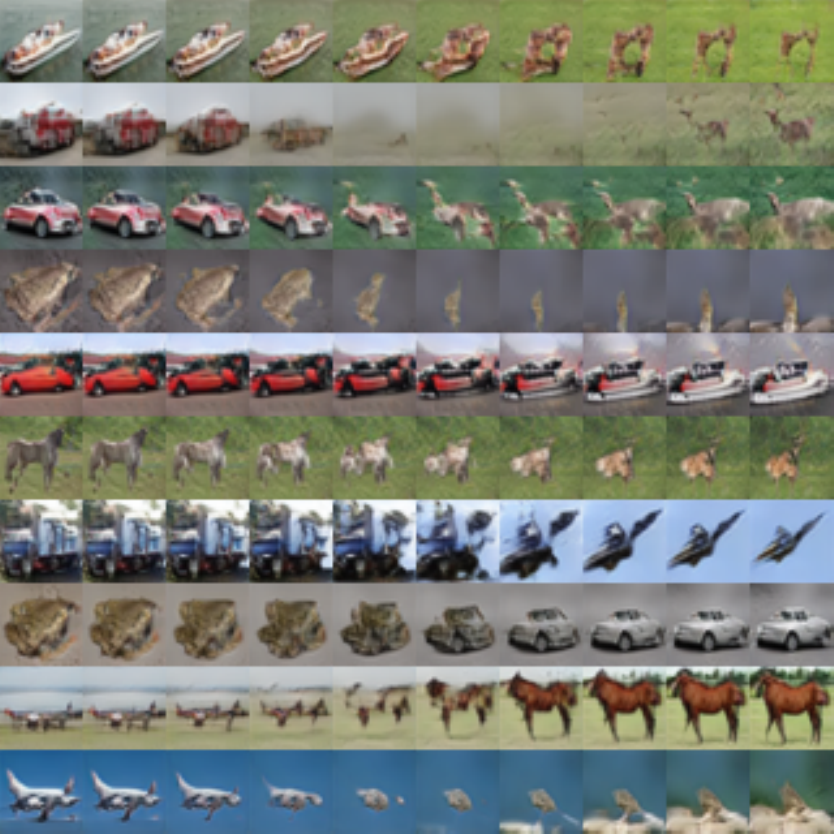}\hspace{1mm}}
    \subfloat[SNGAN-ADD]{\includegraphics[width=0.495\linewidth]{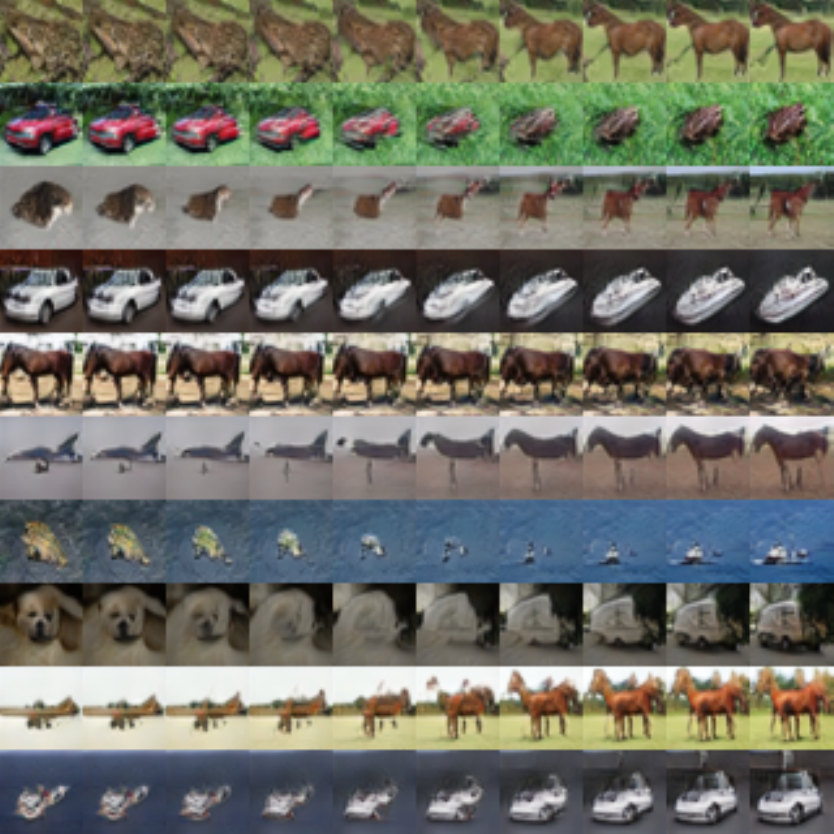}}\\
    \subfloat[SNGAN-\modelnamespade]{\includegraphics[width=0.495\linewidth]{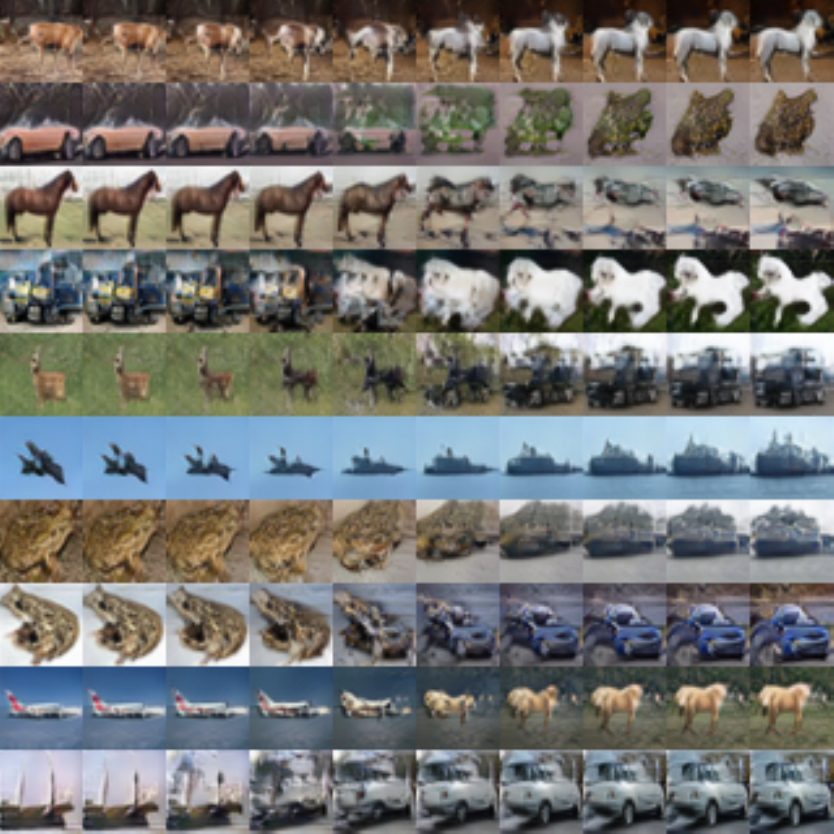}\hspace{1mm}}
    \subfloat[SNGAN-\modelnameMVP]{\includegraphics[width=0.483\linewidth]{qualitative-class_interp_cifar10.pdf}}
\caption{Inter-class linear interpolations across different methods. In inter-class interpolation, the class labels of the leftmost and rightmost images are one-hot vectors, while the rest are interpolated in-between; the resulting images are visualized. Many of the intermediate images in SNGAN-CONC and SNGAN-ADD are either blurry or not realistic. On the contrary, in \modelnamespade{} and \modelnameMVP{} the higher-order polynomial expansion results in more realistic intermediate images. Nevertheless, \modelnameMVP{} results in sharper shapes and images even in the intermediate results when compared to \modelnamespade.}
\label{fig:mvp_comparisons_interclass_cifar10}
\end{figure}

\subsection{Class-conditional generation on house digits}
\label{ssec:mvp_experiment_class_conditional_svhn_suppl}

An experiment on class-conditional generation with SVHN is conducted below. SVHN images include (substantial) blur or other distortions, which insert noise in the distribution to be learned. In addition, some images contain contain a central digit (i.e., based on which the class is assigned), and partial visibility of other digits. Therefore, the generation of digits of SVHN is challenging for a generator without activation functions between the layers. 

Our framework, e.g., \eqref{eq:mvp_model2_rec}, does not include any activation functions. To verify the expressivity of our framework, we maintain the same setting for this experiment. Particularly, \textbf{the generator does not have activation functions between the layers}; there is only a hyperbolic tangent in the output space for normalization. The generator receives a noise sample and a class as input, i.e., it is a class-conditional polynomial generator. 

The results in Fig.~\ref{fig:mvp_qualitative_class_cond_svhn_linear}(b) illustrate that despite the noise, \modelnameMVP{} learns the distribution. As mentioned in the main paper, our formulation enables both inter-class and intra-class interpolations naturally. In the inter-class interpolation the noise $\sbinvar{I}$ is fixed, while the class $\sbinvar{II}$ is interpolated. In Fig.~\ref{fig:mvp_qualitative_class_cond_svhn_linear}(d) several inter-class interpolations are visualized. The visualization exhibits that our framework is able to synthesize realistic images even with inter-class interpolations.

\begin{figure}[!h]
\centering
    \centering
    \subfloat[Ground-truth samples]{\includegraphics[width=0.245\linewidth]{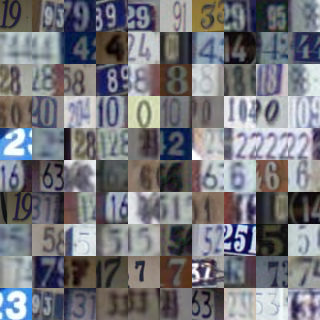}\hspace{1mm}}
    \subfloat[\modelnameMVP]{\includegraphics[width=0.245\linewidth]{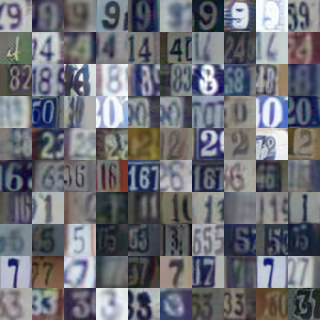}\hspace{1mm}}
    \subfloat[Intra-class interpolation]{\includegraphics[width=0.24\linewidth]{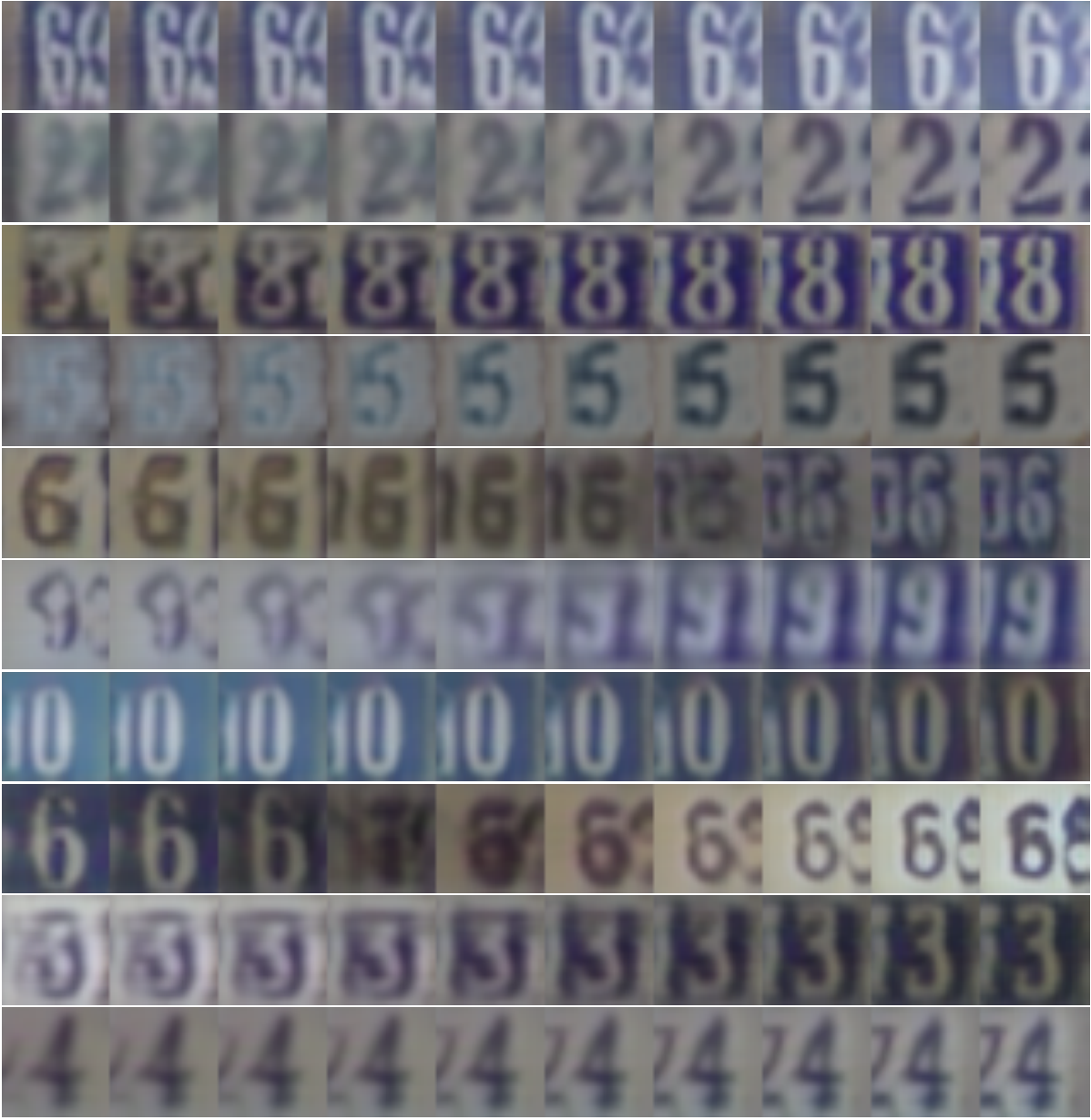}\hspace{1mm}}
    \subfloat[Inter-class interpolation]{\includegraphics[width=0.24\linewidth]{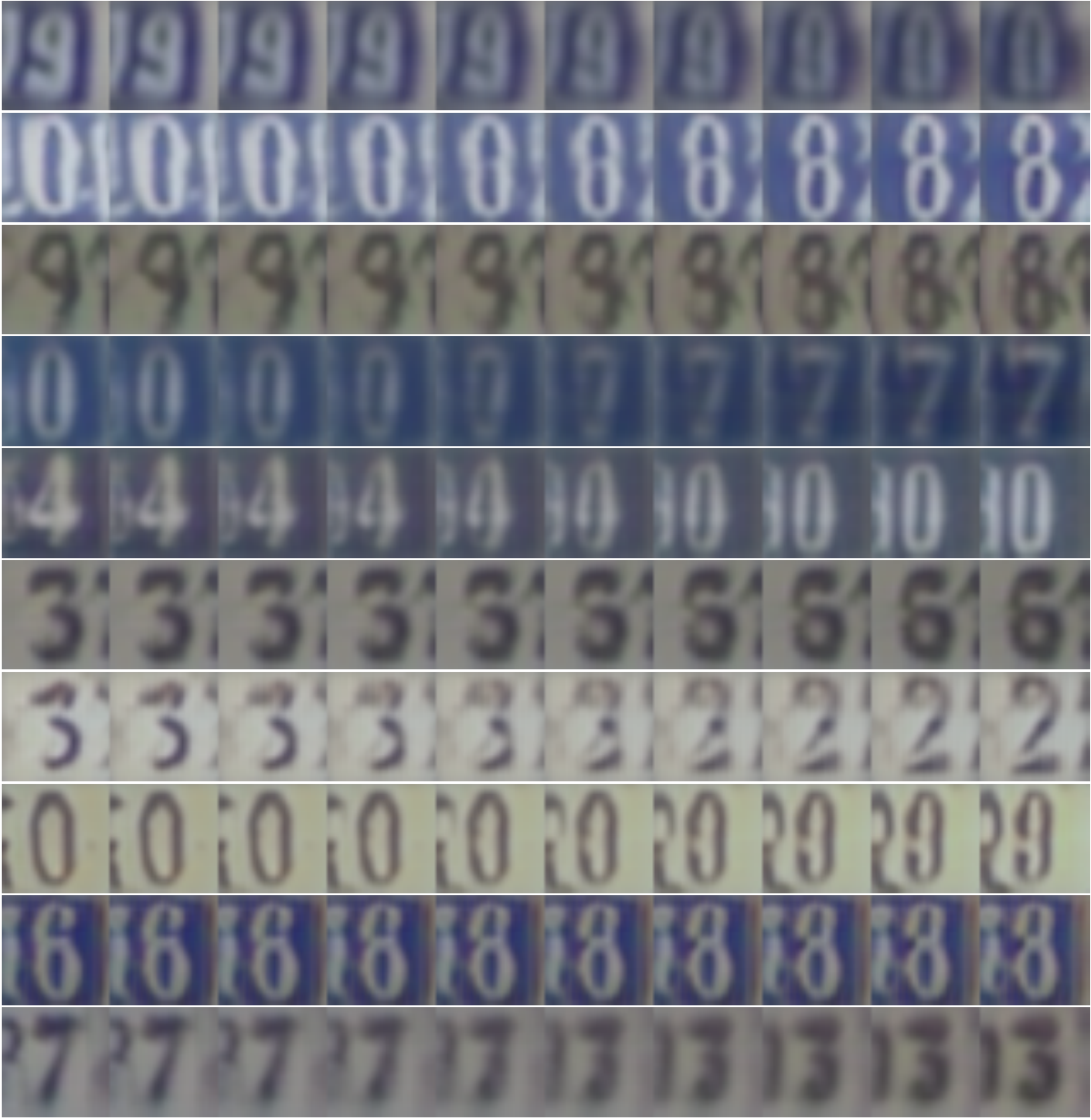}\hspace{1mm}}
\caption{Synthesized images by \modelnameMVP{} in the class-conditional SVHN: (a) Ground-truth samples, (b) Random samples where each row depicts the same class, (c) Intra-class linear interpolation from a source (leftmost image) to the target (rightmost image), (d) inter-class linear interpolation. In inter-class interpolation, the class labels of the leftmost and rightmost images are one-hot vectors, while the rest are interpolated in-between; the resulting images are visualized. In all three cases ((b)-(d)), \modelnameMVP{} synthesizes realistic images.}
\label{fig:mvp_qualitative_class_cond_svhn_linear}
\end{figure}

\subsection{Translation of MNIST digits to SVHN digits}
\label{ssec:mvp_experiment_translation_mnist_svhn_suppl}

An experiment on image translation from the domain of binary digits to house numbers is conducted below. The images of MNIST are used as the source domain (i.e., the conditional variable $\sbinvar{II}$), while the images of SVHN are used as the target domain. The correspondence of the source to the target domain is assumed to be many-to-many, i.e., each MNIST digit can synthesize multiple SVHN images. No additional loss is used, the setting of continuous conditional input from sec.~\ref{ssec:mvp_experiments_continuous} is used. 

The images in Fig.~\ref{fig:mvp_qualitative_mnist2svhn_translation} illustrate that \modelnameMVP{} can translate MNIST digits into SVHN digits. Additionally, for each source digit, there is a significant variation in the synthesized images.

\begin{figure}[!h]
\centering
    \centering
    \includegraphics[width=0.3\linewidth]{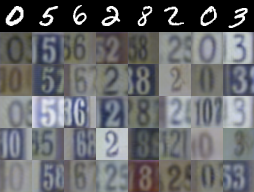}
\caption{Qualitative results on MNIST-to-SVHN translation. The first row depicts the conditional input (i.e., a MNIST digit). The rows 2-6 depict outputs of the \modelnameMVP{} when a noise vector is sampled per row. Notice that for each source digit, there is a significant variation in the synthesized images. }
\label{fig:mvp_qualitative_mnist2svhn_translation}
\end{figure}

\subsection{Translation of edges to images}
\label{ssec:mvp_experiment_translation_edge2im_suppl}

An additional experiment on translation is conducted, where the source domain depicts edges and the target domain is the output image. Specifically, the tasks of edges-to-handbags (on Handbags dataset) and edges-to-shoes (on Shoes dataset) have been selected~\citep{isola2016image}.

In this experiment, the MVP model of sec.~\ref{ssec:mvp_experiments_continuous} is utilized, i.e., a generator without activation functions between the layers. The training is conducted using \emph{only} the adversarial loss. Visual results for both the case of edges-to-handbags and edges-to-shoes are depicted in Fig.~\ref{fig:mvp_qualitative_edges2im_translation}. The first row depicts the conditional input $\sbinvar{II}$, i.e., an edge, while the rows 2-6 depict the synthesized images. Note that in both the case of handbags and shoes there is significant variation in the synthesized images, while they follow the edges provided as input.

\begin{figure}[h]
\centering
    \centering
    \subfloat[edges-to-handbags]{\includegraphics[width=0.49\linewidth]{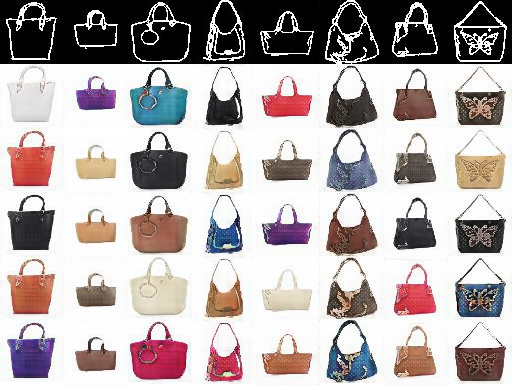}\hspace{3mm}}
    \subfloat[edges-to-shoes]{\includegraphics[width=0.49\linewidth]{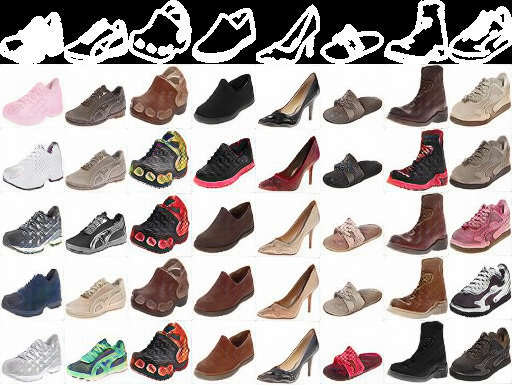}}
\caption{Qualitative results on edges-to-image translation. The first row depicts the conditional input (i.e., the edges). The rows 2-6 depict outputs of the \modelnameMVP{} when we vary $\binvar_I$. Notice that for each edge, there is a significant variation in the synthesized images.}
\label{fig:mvp_qualitative_edges2im_translation}
\end{figure}

\subsection{Multiple conditional inputs in attribute-guided generation}
\label{ssec:mvp_experiment_multiple_discrete_conditional_input_suppl}

Frequently, more than one type of input conditional inputs are available. Our formulation can be extended beyond two input variables (sec.~\ref{sec:mvp_method_many_inputs_suppl}); we experimentally verify this case. The task selected is attribute-guided generation trained on images of Anime characters. Each image is annotated with respect to the color of the eyes ($6$ combinations) and the color of the hair ($7$ combinations).

Since \modelnamespade{} only accepts a single conditional variable, we should concatenate the two attributes in a single variable. We tried simply concatenating the attributes directly, but this did not work well. Instead, we can use the total number of combinations, which is the product of the individual attribute combinations, i.e., in our case the total number of combinations is $42$. Obviously, this causes `few' images to belong in each unique combination, i.e., there are $340$ images on average that belong to each combination. On the contrary, there are $2380$ images on average for each eye color. 

\modelnamespade{} and \modelnamePI{} are trained by using the two attributes in a single combination, while in our case, we consider the multiple conditional variable setting. In each case, only the generator differs depending on the compared method. In Fig.~\ref{fig:mvp_qualitative_attribute_guided_generation_anime_cond_linear} few indicative images are visualized for each method; each row depicts a single combination of attributes, i.e., hair and eye color. Notice that \modelnamespade{} results in a single image per combination, while in \modelnamePI-SINCONC there is considerable repetition in each case. The single image in \modelnamespade{} can be explained by the lack of higher-order correlations with respect to the noise variable $\sbinvar{I}$.

In addition to the diversity of the images per combination, an image from every combination is visualized in Fig.~\ref{fig:mvp_qualitative_attribute_guided_generation_anime_cond_linear_attrib_matrix}. \modelnameMVP{} synthesizes more realistic images than the compared methods of \modelnamePI-SINCONC and \modelnamespade.

\begin{figure*}
\centering
    \centering
    \subfloat[\modelnamePI-SINCONC]{\includegraphics[width=0.32\linewidth]{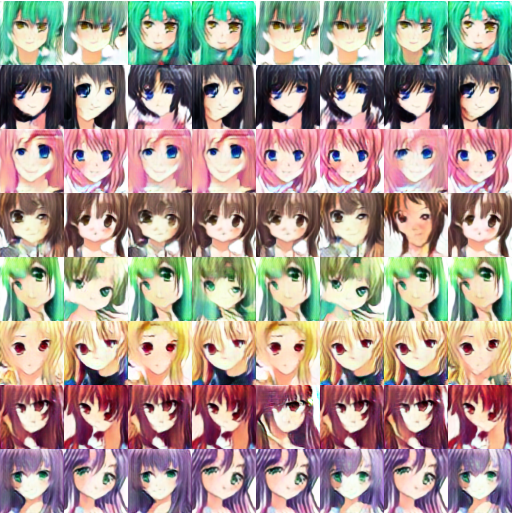}\hspace{1mm}}
    \subfloat[\modelnamespade]{\includegraphics[width=0.32\linewidth]{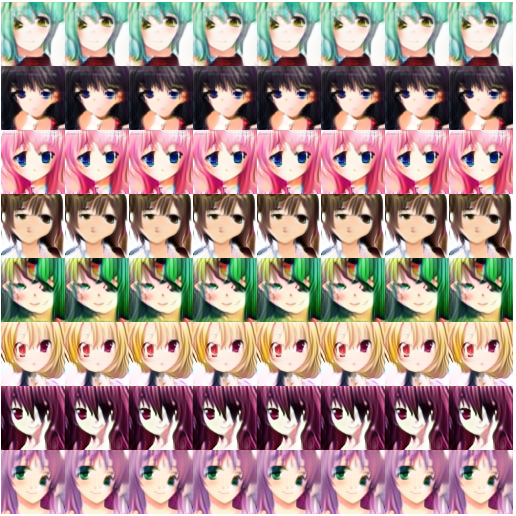}\hspace{1mm}}
    \subfloat[\modelnameMVP]{\includegraphics[width=0.32\linewidth]{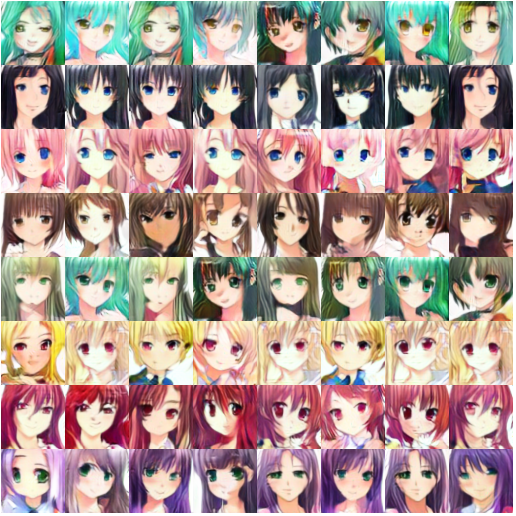}}
\caption{Each row depicts a single combination of attributes, i.e., hair and eye color. Please zoom-in to check the finer details. The method of \modelnamespade{} synthesizes a single image per combination. \modelnamePI-SINCONC synthesizes few images, but not has many repeated elements, while some combinations result in unrealistic faces, e.g., the $5^{th}$ or the $7^{th}$ row. On the contrary, \modelnameMVP{} synthesizes much more diverse images for every combination.}
\label{fig:mvp_qualitative_attribute_guided_generation_anime_cond_linear}
\end{figure*}

\begin{figure*}
\centering
    \centering
    \subfloat[\modelnamePI-SINCONC]{\includegraphics[width=0.32\linewidth]{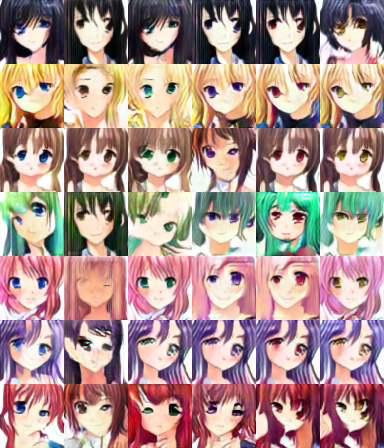}\hspace{2mm}}
    \subfloat[\modelnamespade]{\includegraphics[width=0.32\linewidth]{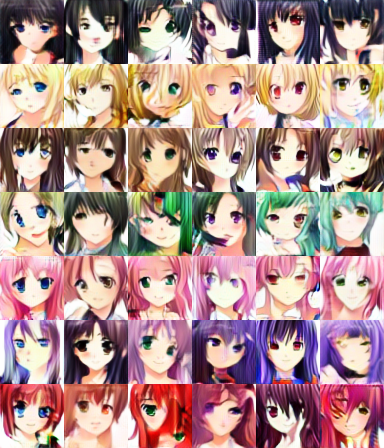}\hspace{2mm}}
    \subfloat[\modelnameMVP]{\includegraphics[width=0.32\linewidth]{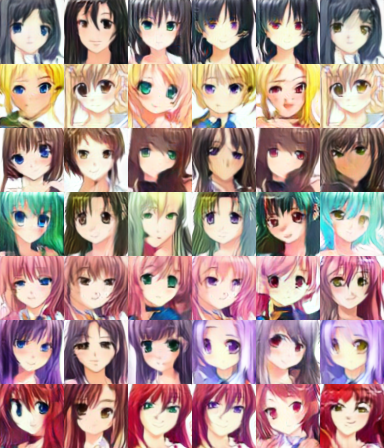}}
\caption{Each row depicts a single hair color, while each column depicts a single eye color. \modelnamespade{} results in some combinations that violate the geometric structure of the face, e.g., $3^{rd}$ column in the last row. Similarly, in \modelnamePI-SINCONC some of the synthesized images are unrealistic, e.g., penultimate row. In both \modelnamespade{} and \modelnamePI-SINCONC, in some cases the eyes do not have the same color. \modelnameMVP{} synthesizes images that resemble faces for every combination.}
\label{fig:mvp_qualitative_attribute_guided_generation_anime_cond_linear_attrib_matrix}
\end{figure*}

\subsection{Multiple conditional inputs in class-conditional super-resolution}
\label{ssec:mvp_experiment_mixed_conditional_input_suppl}

\begin{wrapfigure}[13]{r}{0.2\textwidth}  
\vspace{-10mm}
  \begin{center}
    \includegraphics[width=0.85\linewidth]{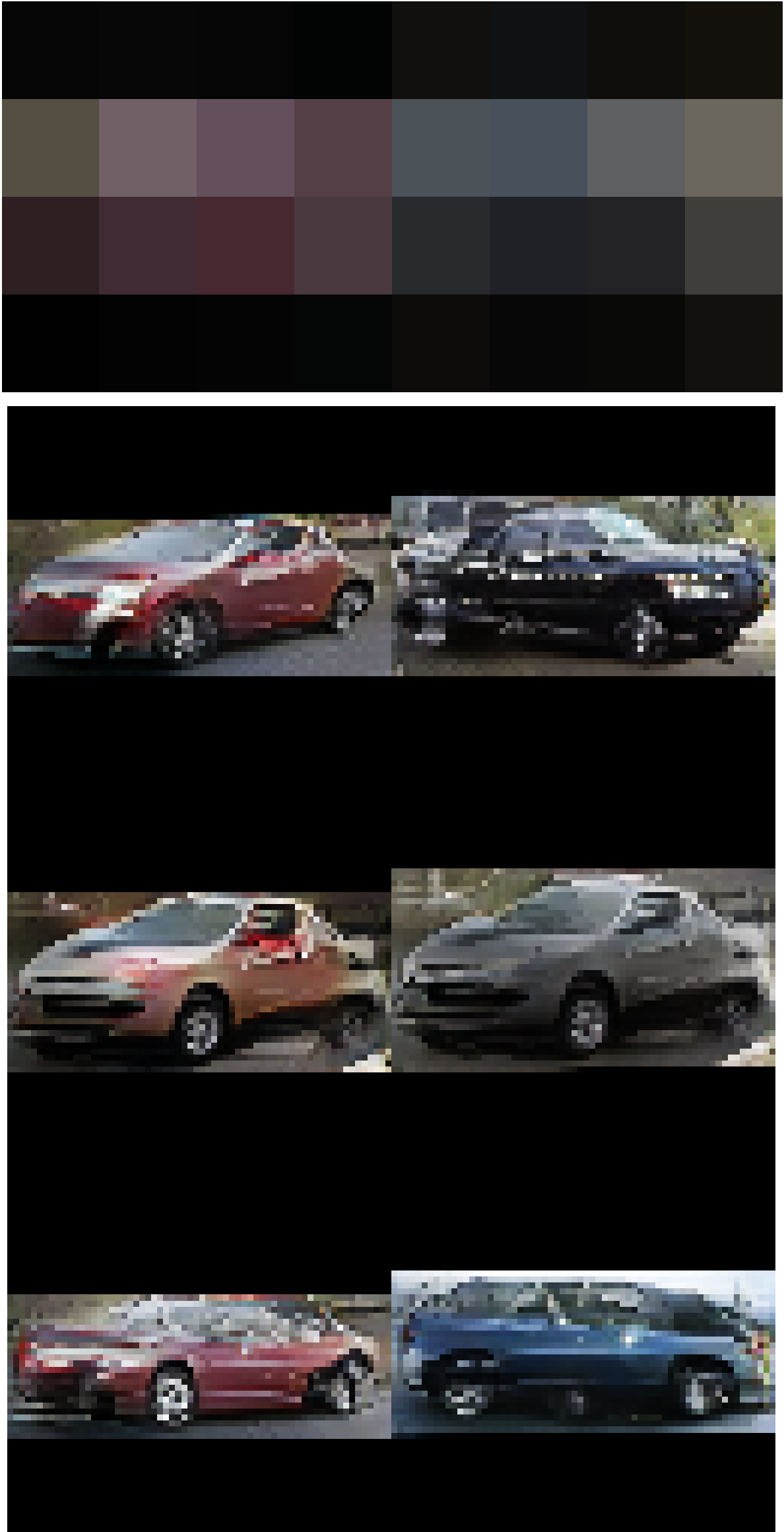}
  \end{center}
  \caption{Three-variable input generative model.} 
  \label{fig:mvp_qualitative_multiple_cond}
\end{wrapfigure}

We extend the previous experiment with multiple conditional variables to the case of class-conditional super-resolution. The first conditional variable captures the class label, while the second conditional variable captures the low-resolution image.

We use the experimental details of sec.~\ref{ssec:mvp_experiments_continuous} in super-resolution $8\times$. In Fig.~\ref{fig:mvp_qualitative_multiple_cond}, we visualize how for each low-resolution image the results differ depending on the randomly sampled class label. The FID in this case is $53.63$, which is similar to the previous two cases. Class-conditional super-resolution (or similar tasks with multiple conditional inputs) can be of interest to the community and \modelnameMVP{} results in high-dimensional images with large variance.

\subsection{Improve diversity with regularization}
\label{ssec:mvp_experiment_translation_edge2im__with_diversity_suppl}

As emphasized in sec.~\ref{sec:mvp_difference_from_other_diverse_generation_techniques_suppl}, various methods have been utilized for synthesizing more diverse images in conditional image generation tasks. A reasonable question is whether our method can be used in conjunction with such methods, since it already synthesizes diverse results. Our hypothesis is that when \modelnameMVP{} is used in conjunction with any diversity-inducing technique, it will further improve the diversity of the synthesized images. To assess the hypothesis, we conduct an experiment on edges to images that is a popular benchmark in such diverse generation tasks~\citep{zhu2017toward, yang2019diversity}. 

The plug-n-play regularization term of \citet{yang2019diversity} is selected and added to the GAN loss during the training. The objective of the regularization term $\mathcal{L}_{reg}$ is to maximize the following term:
\begin{equation}
    \mathcal{L}_{reg} = \min(\frac{\left|| G(\sbinvar{I, 1}, \sbinvar{II}) - G(\sbinvar{I, 2}, \sbinvar{II})\right||_{1}}{\left|| \sbinvar{I, 1} - \sbinvar{I, 2}\right||_{1}}, \tau)
    \label{eq:mvp_diversity_loss}
\end{equation}
where $\tau$ is a predefined constant, $\sbinvar{I, 1}, \sbinvar{I, 2}$ are different noise samples. The motivation behind this term lies in encouraging the generator to produce outputs that differ when the input noise samples differ. In our experiments, we follow the implementation of the original paper with $\tau=10$.

The regularization loss of \eqref{eq:mvp_diversity_loss} is added to the GAN loss; the architecture of the generator remains similar to sec.~\ref{ssec:mvp_experiment_translation_edge2im_suppl}. The translation task is edges-to-handbags (on Handbags dataset) and edges-to-shoes (on Shoes dataset). In Fig.~\ref{fig:mvp_qualitative_edges2im_translation_diversity_regularization} the synthesized images are depicted. The regularization loss causes more diverse images to be synthesized (i.e., when compared to the visualization of Fig.~\ref{fig:mvp_qualitative_edges2im_translation} that was trained using only the adversarial loss). For instance, in both the shoes and the handbags, new shades of blue are now synthesized, while yellow handbags can now be synthesized. 

The empirical results validate the hypothesis that our model can be used in conjunction with diversity regularization losses in order to improve the results. Nevertheless, the experiment in sec.~\ref{ssec:mvp_experiment_translation_edge2im_suppl} indicates that a regularization term is not necessary to synthesize images that do not ignore the noise as feed-forward generators had previously.

\begin{figure}[h]
\centering
    \centering
    \subfloat[edges-to-handbags]{\includegraphics[width=0.49\linewidth]{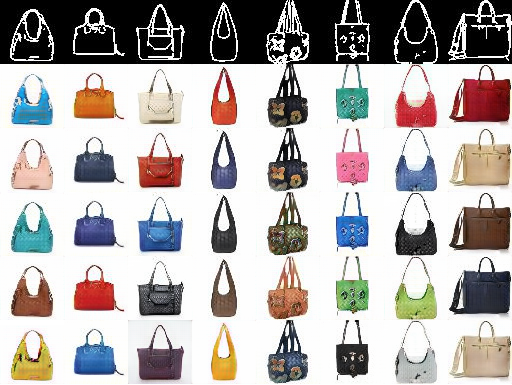}\hspace{3mm}}
    \subfloat[edges-to-shoes]{\includegraphics[width=0.49\linewidth]{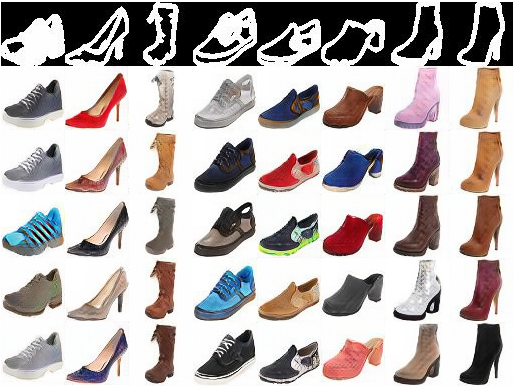}}
\caption{Qualitative results on edges-to-image translation with regularization loss for diverse generation (sec.~\ref{ssec:mvp_experiment_translation_edge2im__with_diversity_suppl}). The first row depicts the conditional input (i.e., the edges). The rows 2-6 depict outputs of the \modelnameMVP{} when we vary $\binvar_I$. Diverse images are synthesized for each edge. The regularization loss results in `new' shades of blue to emerge in the synthesized images in both the shoes and the handbags cases.}
\label{fig:mvp_qualitative_edges2im_translation_diversity_regularization}
\end{figure}

\subsection{Generation of unseen attribute combinations}
In this section, we highlight how the proposed CoPE-VAE compares to the other baselines for the task of generating unseen attribute combinations. Following the benchmark in \citep{georgopoulos2020multilinear} we compare to cVAE \citep{sohn2015learning}, VampPrior \citep{tomczak2017vae} and MLC-VAE-CP/T \citep{georgopoulos2020multilinear}. The results in Figure \ref{fig:cope_vae_unseen_combinations} and Table \ref{tab:missing_att_comp} showcase the efficacy of our method in recovering the unseen attribute combination of ("smiling", "female"). In particular, the quantitative comparison in Table \ref{tab:missing_att_comp} shows that our method is on par or outperforms the baseline methods in both attribute synthesis and attribute transfer. The quantitative results were obtained using attribute classifiers trained on CelebA for gender and smile.

\begin{figure}[h]
\centering
    \includegraphics[width=0.7\linewidth]{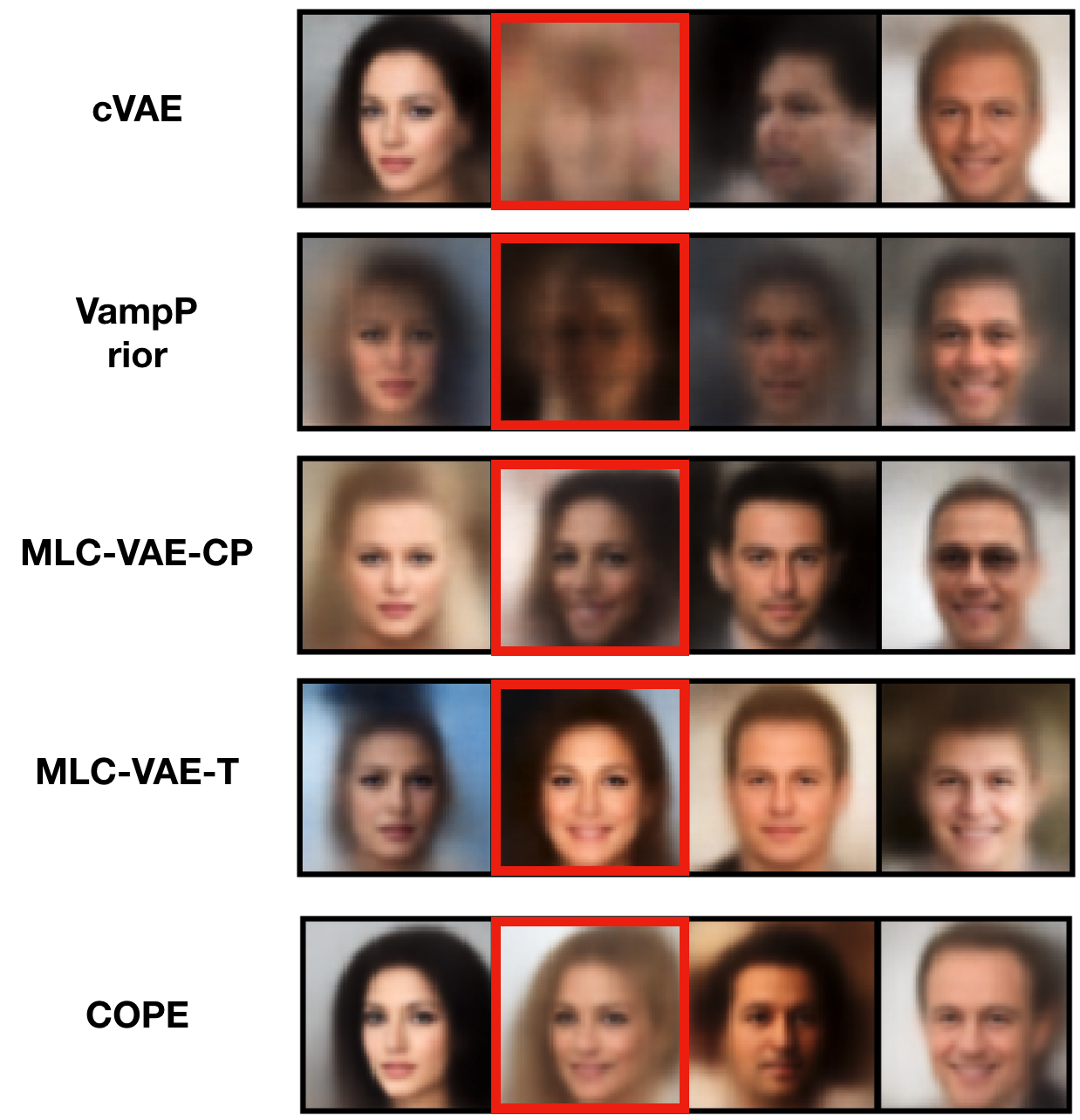}
\caption{Qualitative comparison on CelebA. The missing combination, i.e., ("Smiling", "Female"), is in the red rectangle. Results for cVAE, VampPrior, MLC-VAE-CP and MLC-VAE-T are taken directly from \citep{georgopoulos2020multilinear}.}
\label{fig:missing_atts}
\end{figure}

\begin{table*}[h]

\centering
\vspace{0pt}
 
         \vspace{0pt}
         \begin{tabular}{|c | c | c || c | c |}
             \hline
             &\multicolumn{2}{c||}{Attribute synthesis}  &\multicolumn{2}{c|}{Attribute transfer}\\ \hline  
             Model & Acc. Gender (\%) $\uparrow$ & Acc. Smile (\%) $\uparrow$ & Acc. Gender (\%) $\uparrow$ & Acc. Smile (\%) $\uparrow$\\\hline
             
             cVAE & $18$& $7.6$ & $23.1$ & $9$\\ \hline
             cVampPrior & $17.1$&$2.8$ &$44.3$ &$5.8$\\\hline 
             MLC-VAE-CP & $96.4$ & $94$ & $95$ & $90.6$ \\\hline
             MLC-VAE-T & $99.4$ & $93.5$& $99.4$ & $91.5$ \\\hline
             CoPE-VAE & $99.9$ & $92.6$ & $99.5$ & $92.38$ \\\hline
         \end{tabular}
         \caption{Quantitative comparison on CelebA using attribute classifiers. The reported results are \textbf{only for the unseen attribute combination}, i.e., ("Smiling", "Female"). Results for cVAE, VampPrior, MLC-VAE-CP and MLC-VAE-T are taken directly from \citep{georgopoulos2020multilinear}.}
         \label{tab:missing_att_comp}

\end{table*} \section{Diverse generation techniques and its relationship with \modelnameMVP}
\label{sec:mvp_difference_from_other_diverse_generation_techniques_suppl}

One challenge that often arises in conditional data generation is that one of the variables gets ignored by the generator~\citep{isola2016image}. This has been widely acknowledged in the literature, e.g., \citet{zhu2017toward} advocates that it is hard to utilize a simple architecture, like \citet{isola2016image}, with noise. A similar conclusion is drawn in InfoGAN~\citep{chen2016infogan} where the authors explicitly mention that additional losses are required, otherwise the generator is `free to ignore' the additional variables. To mitigate this, a variety of methods have been developed. We summarize the most prominent methods from the literature, starting from image-to-image translation methods:
\begin{itemize}
    \item BicycleGAN~\citep{zhu2017toward} proposes a framework that can synthesize diverse images in image-to-image translation. The framework contains 2 encoders, 1 decoder and 2 discriminators. This results in multiple loss terms (e.g., eq.9 of the paper). Interestingly, the authors utilize a separate training scheme for the encoder-decoder and the second encoder as training together 'hides the information of the latent code without learning meaningful modes'. 
    
    \item \citet{almahairi2018augmented} augment the deterministic mapping of CycleGAN~\citep{zhu2017unpaired} with a marginal matching loss. The framework learns diverse mappings utilizing the additional encoders. The framework includes 4 encoders, 2 decoders and 2 discriminators.
    
    \item MUNIT~\citep{huang2018multimodal} focuses on diverse generation in unpaired image-to-image translation. MUNIT demonstrates impressive translation results, while the inverse translation is also learnt simultaneously. That is, in case of edges-to-shoes, the translation shoes-to-edges is also learnt during the training. The mapping learnt comes at the cost of multiple network modules. Particularly, MUNIT includes 2 encoders, 2 decoders, 2 discriminators for learning. This also results in multiple loss terms (e.g., eq.5 of the paper) along with additional hyper-parameters and network parameters.
    
    \item Drit++~\citep{lee2020drit++} extends unpaired image-to-image translation with disentangled representation learning, while they allow multi-domain image-to-image translations. Drit++ uses 4 encoders, 2 decoders, 2 discriminators for learning. Similarly to the previous methods, this results in multiple loss terms (e.g., eq.6-7 of the paper) and additional hyper-parameters.
    
    \item \citet{choi2020stargan} introduce a method that supports multiple target domains. The method includes four modules: a generator, a mapping network, a style encoder and a discriminator. All modules (apart from the generator) include domain-specific sub-networks in case of multiple target domains. To ensure diverse generation, \citet{choi2020stargan} utilize a regularization loss (i.e., eq. 3 of the paper), while their final objective consists of multiple loss terms.
\end{itemize}

The aforementioned frameworks contain additional network modules for training, which also results in additional hyper-parameters in the loss-function and the network architecture. Furthermore, the frameworks focus exclusively on image-to-image translation and not all conditional generation cases, e.g., they do not tackle class-conditional or attribute-based generation.

Using regularization terms in the loss function has been an alternative way to achieve diverse generation. \citet{mao2019mode, yang2019diversity} propose simple regularization terms that can be plugged into any architecture to encourage diverse generation. \citet{lee2019harmonizing} propose two variants of a regularization term, with the `more stable variant' requiring additional network modules.

Even though our goal is not to propose a framework for diverse generation, the synthesized images of \modelnameMVP{} contain variation when the noise is changed. However, more diverse results can be exhibited through a dedicated diverse generation technique.   
We emphasize that our method can be used in conjunction with many of the aforementioned techniques to obtain more diverse examples. We demonstrate that this is possible in an experiment in sec.~\ref{ssec:mvp_experiment_translation_edge2im__with_diversity_suppl}.

\end{document}